\documentclass[twoside]{article}

\usepackage[preprint]{aistats2026}
\usepackage[round]{natbib}
\setcitestyle{authoryear, open={(},close={)}}

\usepackage{amsmath,amssymb,amsfonts}
\usepackage{bm}
\usepackage{mathtools}
\usepackage{graphicx}
\usepackage{subcaption}
\usepackage{booktabs}
\usepackage{colortbl}
\usepackage{xcolor}
\usepackage{enumitem}
\usepackage{siunitx}
\usepackage{float}

\definecolor{bestOurs}{RGB}{0,0,255}   
\definecolor{bestOther}{RGB}{255,0,0}  
\definecolor{btmgray}{gray}{0.9}       
\newcommand{\bestOurs}[1]{\textbf{\textcolor{bestOurs}{#1}}}
\newcommand{\bestOther}[1]{\textbf{\textcolor{bestOther}{#1}}}
\newcolumntype{B}{!{\vrule width 1pt}}

\usepackage{multirow}
\usepackage{array}
\usepackage{microtype}
\usepackage{algorithm}
\usepackage{algorithmic}
\usepackage{amsthm}

\usepackage{hyperref}
\hypersetup{
    colorlinks=true,  
    linkcolor=blue,   
    citecolor=blue,   
    urlcolor=blue     
}
\usepackage{xr-hyper}

\newtheorem{theorem}{Theorem}

\begin{document}

%

%





\runningauthor{Nganjimi et al.}

\twocolumn[
\aistatstitle{Improving Clinical Dataset Condensation with Mode Connectivity–based Trajectory Surrogates}

\aistatsauthor{
Pafue Christy Nganjimi$^{1}$ \And Andrew Soltan$^{1}$ \And Danielle Belgrave$^{2}$ \And
Lei Clifton$^{3}$ \AND David A. Clifton$^{1,4}$ \And Anshul Thakur$^{1}$
}

\aistatsaddress{\\
$^{1}$Institute of Biomedical Engineering, University of Oxford, UK \\
$^{2}$GlaxoSmithKline, London, UK \\
$^{3}$Nuffield Department of Primary Care Health Sciences, University of Oxford, UK \\
$^{4}$Oxford-Suzhou Institute of Advanced Research (OSCAR), Suzhou, China
}
]

\begin{abstract}
  Dataset condensation (DC) enables the creation of compact, privacy-preserving synthetic datasets that can match the utility of real patient records, supporting democratised access to highly regulated clinical data for developing downstream clinical models. State-of-the-art DC methods supervise synthetic data by aligning the training dynamics of models trained on real and those trained on synthetic data, typically using full stochastic gradient descent (SGD) trajectories as alignment targets; however, these trajectories are often noisy, high-curvature, and storage-intensive, leading to unstable gradients, slow convergence, and substantial memory overhead. We address these limitations by replacing full SGD trajectories with smooth, low-loss parametric surrogates, specifically quadratic Bézier curves that connect the initial and final model states from real training trajectories. These mode-connected paths provide noise-free, low-curvature supervision signals that stabilise gradients, accelerate convergence, and eliminate the need for dense trajectory storage. We theoretically justify Bézier-mode connections as effective surrogates for SGD paths and empirically show that the proposed method outperforms state-of-the-art condensation approaches across five clinical datasets, yielding condensed datasets that enable clinically effective model development.

\end{abstract}

\section{INTRODUCTION}
\label{sec:intro}

\begin{figure*}[t]
    \centering
    \begin{subfigure}[t]{0.32\linewidth}
        \centering
        \includegraphics[width=\linewidth]{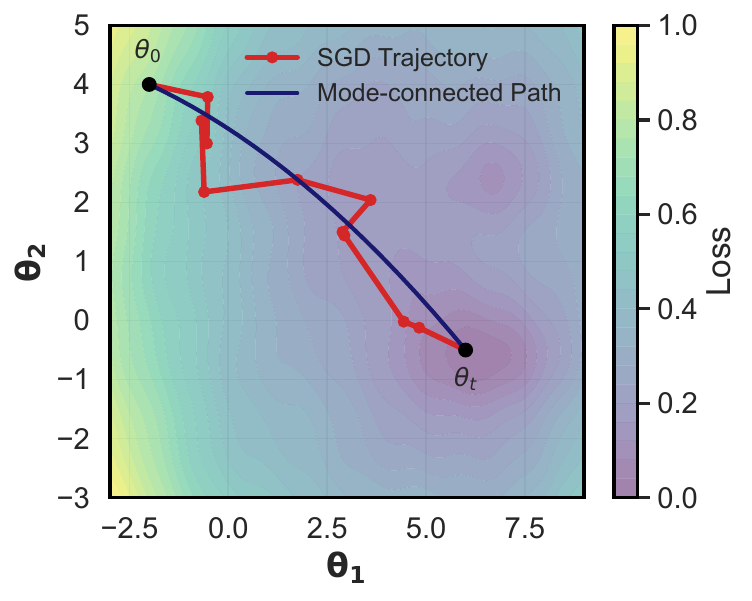}
        \subcaption{Parameter-space Paths}
        \label{fig:tm_param_space}
    \end{subfigure}
    \hfill
    \begin{subfigure}[t]{0.32\linewidth}
        \centering
        \includegraphics[width=\linewidth]{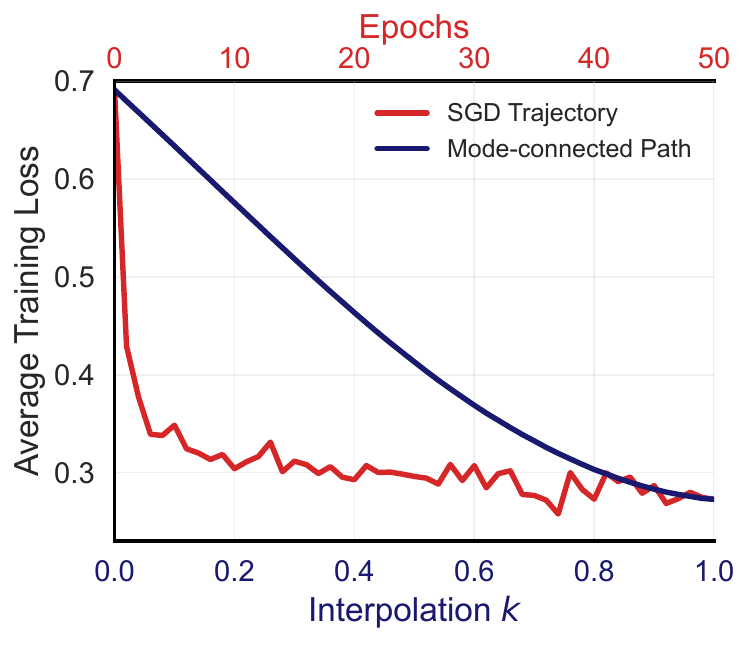}
        \subcaption{Training Set Loss Along Paths}
        \label{fig:tm_loss_path}
    \end{subfigure}
    \hfill
    \begin{subfigure}[t]{0.32\linewidth}
        \centering
        \includegraphics[width=\linewidth]{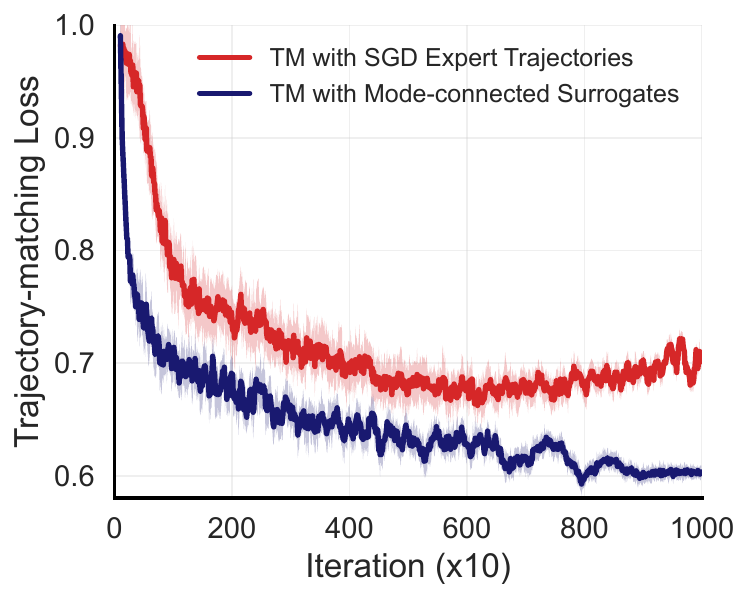}
        \subcaption{Condensation Loss Convergence}
        \label{fig:tm_convergence}
    \end{subfigure}
    \caption{
        Illustration of the key differences between traditional SGD trajectories and mode-connected paths in TM.  
        (a) SGD trajectories are noisy and require many intermediate checkpoints, whereas mode-connected paths give smooth, direct connections using only the start and end models.  
        (b) The average loss of the training set fluctuates and has high curvature along SGD trajectories, while mode-connected paths yield a stable, smoothly decreasing profile.  
        (c) Mode-connected surrogates accelerate optimisation and reach lower trajectory-matching loss than raw SGD trajectories.
    }
    \label{fig:tm_vs_sgd}
\end{figure*}

Clinical artificial intelligence (AI) advances not by algorithms alone but by access to large, well-characterised datasets drawn from electronic health records (EHRs), disease registries, and biobanks \citep{rajkomar2019machine}. Yet access to clinical data is tightly constrained by privacy regulations and institutional governance frameworks designed to protect sensitive patient data \citep{thakur2024data,shabani2019gdpr}. These constraints slow methodological innovation by limiting opportunities to develop and rigorously evaluate models on real-world cohorts, undermining generalisability and clinical robustness \citep{topol2019high}. Dataset condensation (DC) addresses this gap by synthesising a compact, task-preserving proxy that approaches full data performance at a fraction of the storage and training cost \citep{li2025dataset}. The condensation objective typically yields synthetic samples that act as latent summaries of many underlying real samples, weakening any one-to-one correspondence with individual patients \citep{wang2023medical}. When paired with differential privacy, DC provides formal guarantees that patient-level information cannot be inferred \citep{carlini2022no}.

Trajectory matching (TM) constitutes a prominent class of methods that define the current state of the art in DC \citep{cazenavette2022dataset, liu2023tesla, du2023ftd}. The central idea is to optimise the synthetic dataset by explicitly aligning the parameter trajectories induced when training a model on the synthetic data with those obtained from training on the real data. This alignment objective ensures that the condensed dataset not only preserves discriminative information but also replicates the training dynamics of the original data, thereby improving generalisation and downstream performance \citep{guo2024datm}. Despite their empirical success, TM methods face several practical limitations. They typically require storing and sampling from multiple intermediate checkpoints along each training trajectory, incurring substantial storage overhead. Moreover, stochastic gradient descent (SGD) trajectories are inherently noisy and frequently traverse high-curvature regions of the loss landscape, leading to high variability and unstable supervision signals for the condensation process (see Fig.~\ref{fig:tm_vs_sgd}a–b). These issues can lead to slower optimisation, degraded alignment quality, and limited generalisation across architectures.

To address these limitations, we replace full SGD trajectories with smooth, low-loss parametric surrogates computed via mode connectivity \citep{garipov2018loss,thakur2025optimising}. For each real-data training trajectory, we learn a mode-connected path by training a quadratic Bézier curve in parameter space that links the initial (high-loss) and final (low-loss) model states; the curve is defined by a single control point, which we optimise to minimise the average training loss along the path. Crucially, this path preserves the same endpoints and overall optimisation direction as the original trajectory while smoothing away stochastic, high-curvature detours, yielding a fixed, low-curvature, low-loss supervision target once fitted. Optimisation-wise, removing stochastic sampling noise from the target and enforcing low curvature produce a well-conditioned, low-variance signal in which gradients vary smoothly along the path. This improves conditioning, stabilises updates, and accelerates convergence, reducing computational cost without sacrificing alignment fidelity. From a storage standpoint, each trajectory requires storing only three points: the initial checkpoint, the final checkpoint, and the learned control point. The intermediate states can be computed on the fly (see Section~\ref{subsec:surr-hyp}), eliminating the need to store multiple intermediate checkpoints and yielding order-of-magnitude savings.

Building on these mode-connected surrogates, we formulate a DC framework that supervises synthetic data via alignment to smooth, low-loss trajectories. Rather than aligning to noisy SGD paths, we sample pairs of nearby points along the surrogate curves and optimise the synthetic dataset so that a model trained on it advances between them. This retains the trajectory-level supervision signal while improving stability, reducing variance, and eliminating the need to store full optimisation traces. The major contributions of this work are listed below: 

\begin{itemize}[topsep=0pt]
    \item We introduce mode connectivity–based trajectory surrogates (quadratic Bézier curves) as compact, low-curvature replacements for full SGD paths in TM-based DC framework. This proposed mechanism can be employed in any existing TM framework.
    \item We theoretically justify the use of mode connectivity, specifically quadratic Bézier curves, as an effective surrogate for full optimisation trajectories, explaining why it preserves essential training signals while avoiding SGD stochasticity.
    \item We conduct an extensive empirical evaluation on five real-world clinical datasets, demonstrating that the proposed approach produces clinically useful condensed datasets that support reliable downstream model development. It achieves performance comparable to real-data training and consistently outperforms state-of-the-art baselines across the evaluated clinical datasets.
\end{itemize}

\section{RELATED WORKS}
\label{sec:rel_works}

\textsc{Dataset Condensation:} 
Dataset Distillation \citep{wang2018dataset} pioneered synthetic dataset synthesis through nested bilevel optimisation, addressing the computational burden of training on massive real-world datasets by creating dramatically smaller yet equally effective alternatives. Gradient matching approaches \citep{zhao2021dataset} improved computational efficiency by avoiding expensive inner-loop unrolling, while soft-label extensions \citep{sucholutsky2019soft, bohdal2020flexible} explored label optimisation. Kernel-based methods \citep{nguyen2021dataset, nguyen2022dataset, jacot2018neural} reformulated the problem using neural tangent kernel theory, offering closed-form solutions but with heavy computational costs. Distribution matching methods \citep{zhao2023dataset, wang2022cafe, zhao2023improved, zhang2024m3d} achieved efficiency gains by aligning feature statistics without bilevel optimisation, though often requiring larger synthetic sets to achieve strong performance. Recent image-focussed decoupling approaches like Squeeze, Recover and Relabel \citep{yin2023sre2l} and Realistic, Diverse, and Efficient Dataset Distillation \citep{sun2023rded} separate synthesis from training optimisation. While effective on vision datasets, these methods transfer poorly to structured clinical data, where performance gaps remain significant.

\textsc{Trajectory Matching:} 
TM narrows this gap by supervising synthetic data with the training dynamics of real data. Matching Training Trajectories (MTT) \citep{cazenavette2022dataset} first aligned long-range trajectories, capturing richer information than stepwise gradient matching. Extensions such as Flat Trajectory Distillation (FTD) \citep{du2023ftd}, Difficulty-Aligned Trajectory Matching (DATM) \citep{guo2024datm}, and TrajEctory matching with Soft Label Assignment (TESLA) \citep{liu2023tesla} introduce curvature regularisation, difficulty-based curricula, and scalable soft-label assignment respectively. Despite strong results, TM methods depend on dense trajectories—tens to hundreds of checkpoints—and inherit the noise and curvature of SGD paths, inflating storage and introducing instability that limits clinical use.

\textsc{Mode Connectivity:}
Mode connectivity studies the geometry of neural loss landscapes by constructing smooth, low-loss paths between trained models, most commonly parameterised as quadratic Bézier curves \citep{garipov2018loss, draxler2018essentially, izmailov2018averaging}. These paths preserve endpoint performance while bypassing high-loss regions, and have been used to analyse generalisation and reparameterisation. In clinical machine learning, they have supported incremental learning and mitigated distribution shift \citep{thakur2023clinical}. Requiring only two endpoints and a single control point, mode connections provide compact yet expressive representations of optimisation paths.

\textsc{Comparison With Proposed Approach:}
These approaches present distinct trade-offs: non-TM methods are efficient but struggle with structured clinical data; TM captures richer dynamics but suffers from SGD instability and heavy storage requirements; mode connectivity provides compact, smooth paths but remains unexplored for DC. This work bridges these limitations by replacing noisy SGD trajectories with smooth mode-connected surrogates, combining TM's supervisory power with mode connectivity's efficiency for effective clinical data synthesis.

\section{BACKGROUND}
\label{sec:background}

\textsc{Problem Statement:} Given a large real dataset $\mathcal{D} = \{(\boldsymbol{x}_i, \boldsymbol{y}_i)\}_{i=1}^{|\mathcal{D}|}$, where  $\boldsymbol{x}_i \in \mathbb{R}^d$ is an example and $\boldsymbol{y}_i$ its corresponding class label, DC seeks to learn a much smaller synthetic dataset $\tilde{\mathcal{D}} = \{(\boldsymbol{\tilde{x}}_i, \boldsymbol{\tilde{y}}_i)\}_{i=1}^{|\tilde{\mathcal{D}}|}$ with $|\tilde{\mathcal{D}}| \ll |\mathcal{D}|$. The objective is that models trained on $\tilde{\mathcal{D}}$ achieve comparable performance to those trained on $\mathcal{D}$. In practice, only a few synthetic samples per class (often referred to as \emph{instances per class}) are retained, making the overall synthetic dataset orders of magnitude smaller than the original. 

Formally, DC can be expressed as the optimisation problem:
\begin{equation}
\begin{split}
    \tilde{\mathcal{D}}^* &= \arg\min_{\tilde{\mathcal{D}}} \;
    \mathbb{E}_{\boldsymbol{\theta}_0 \sim \mathcal{P}}
    \big[ \mathcal{L}_{\mathrm{val}}(f_{\boldsymbol{\theta}_T}) \big], \\
    &\quad \text{where } \boldsymbol{\theta}_T \leftarrow 
    \mathrm{Train}(f_{\boldsymbol{\theta}_0}, \tilde{\mathcal{D}}).
\end{split}
\label{eq:dc}
\end{equation}

Here, $\boldsymbol{\theta}_0 \sim \mathcal{P}$ denotes parameters drawn from the initialisation distribution, $\boldsymbol{\theta}_T$ are the parameters obtained after training on the synthetic dataset $\tilde{\mathcal{D}}$, and $\mathcal{L}_{\mathrm{val}}$ is the validation loss of the trained model $f_{\boldsymbol{\theta}_T}$. 
The goal is to find a synthetic dataset $\tilde{\mathcal{D}}^*$ that minimises this expected validation loss, ensuring that training from scratch on $\tilde{\mathcal{D}}$ yields models that generalise as well as those trained on the full dataset.

\textsc{Trajectory Matching Framework:} TM supervises synthetic dataset learning using \emph{expert trajectories}, the optimisation checkpoints $\{\boldsymbol{\theta}_k\}_{k=0}^K$ obtained by training a model on the real dataset $\mathcal{D}$. During condensation, a segment of length $M$ is sampled by starting from checkpoint $\boldsymbol{\theta}_k$ and targeting a later point $\boldsymbol{\theta}_{k+M}$. A student model initialised at $\boldsymbol{\tilde{\theta}}_{k,0} = \boldsymbol{\theta}_k$ is trained on the synthetic dataset $\tilde{\mathcal{D}}$ for $N$ steps, and $\tilde{\mathcal{D}}$ is updated to minimise the displacement between the student’s final parameters $\boldsymbol{\tilde{\theta}}_{k,N}$ and the target $\boldsymbol{\theta}_{k+M}$. Repeated alignment across segments teaches $\tilde{\mathcal{D}}$ to guide optimisation from initialisation to convergence. While effective, TM requires storing many intermediate checkpoints per trajectory, creating storage overhead, and its supervision inherits the stochastic, indirect paths of SGD, leading to instability.

\section{PROPOSED METHOD}
\label{sec:prop_meth}
The limitations of TM frameworks raise a central question: does DC require replicating the exact, noisy dynamics of SGD, or can smoother surrogate paths provide equally effective supervision? This section addresses this question by introducing a surrogate-path hypothesis, analysing its theoretical properties, and presenting a condensation framework built on these paths.

\subsection{Surrogate-path Hypothesis}
\label{subsec:surr-hyp}
The core role of TM is to teach synthetic datasets how to guide models from initialisation towards well-performing parameter regions in the loss landscape. This does not necessarily require exact replication of the stochastic, noisy trajectories followed by SGD; what matters is providing reliable guidance between checkpoints. This motivates the hypothesis that \emph{smooth, low-loss paths connecting the initial and final model states---ideally with a monotonically or near-monotonically decreasing loss profile---can serve as effective surrogates for the noisy trajectories followed by SGD}. Analyses of neural network loss landscapes show that multiple such non-linear low-loss paths exist, and can act as stable alternatives to the noisy SGD trajectories used in TM frameworks.

Mode connectivity provides a principled way to realise these surrogate paths. A mode connection, denoted $\gamma_{\boldsymbol{\theta}_A \rightarrow \boldsymbol{\theta}_B}$, is defined as a low-loss path in parameter space between two optima $\boldsymbol{\theta}_A, \boldsymbol{\theta}_B \in \mathbb{R}^N$. Along this path, every intermediate point remains an effective optimum and no loss barriers are encountered, so that the loss decreases monotonically from $\boldsymbol{\theta}_A$ to $\boldsymbol{\theta}_B$. This highlights that mode connections essentially exhibit the characteristics required to be the surrogate trajectories. 

To approximate a training trajectory, the connection must extend beyond optima to link an initial high-loss model state $\boldsymbol{\theta}_0$ with a converged model state $\boldsymbol{\theta}_T$. We therefore employ a parameterised quadratic Bézier curve $\Phi_{\boldsymbol{\phi}}(t)$ \citep{garipov2018loss} as a non-linear mode connection:
\begin{equation}
    \Phi_{\boldsymbol{\phi}}(t) = (1-t)^2 \boldsymbol{\theta}_0 + 2t(1-t)\boldsymbol{\phi} + t^2 \boldsymbol{\theta}_T,
\label{eq:bz}
\end{equation}

where $\boldsymbol{\phi} \in \mathbb{R}^N$ is the trainable control point of the curve. The variable $t \in [0,1]$ determines the position along the path, with $t=0$ corresponding to $\boldsymbol{\theta}_0$ and $t=1$ to $\boldsymbol{\theta}_T$. The curve $\Phi_{\boldsymbol{\phi}}$ is fitted by updating control point $\boldsymbol{\phi}$ to minimise the training loss at sampled points along the path using dataset $\mathcal{D}$, thereby yielding the desired mode connection $\gamma_{\boldsymbol{\theta}_0 \rightarrow \boldsymbol{\theta}_T}$ \citep{thakur2025optimising}:
\begin{equation}
\label{eq:phi_optim}
    \boldsymbol{\phi}^\star = \arg\min_{\boldsymbol{\phi}}
    \mathbb{E}_{t \sim \mathcal{U}(0,1)}
    \big[\mathcal{L}_{\mathrm{train}}(\Phi_{\boldsymbol{\phi}}(t), \mathcal{D})\big].
\end{equation}

The resulting path, $\Phi_{\boldsymbol{\phi}^{\star}}(t)$, is fully determined by just three parameter vectors: the initial state $\boldsymbol{\theta}_0$, the final state $\boldsymbol{\theta}_T$, and the learned control point $\boldsymbol{\phi}^\star$. Hence, we obtain a surrogate for an SGD trajectory that is both noise-free and low-curvature (Figure~\ref{fig:tm_vs_sgd}a), reliably connects high-loss and low-loss regions in parameter space (Figure~\ref{fig:tm_vs_sgd}b), and eliminates the need to store the full trajectory.

\subsection{Theoretical Guarantees of Bézier Surrogates} 
\label{subsec:bez-theory}

Bézier curves provide not only a faithful surrogate to noisy SGD trajectories but also a potentially superior alternative for optimisation. An optimised Bézier path retains the functional behaviour of the trajectory—maintaining low average loss and consistent predictions—while offering two crucial advantages: it is noise-free and exhibits strictly lower curvature. These properties make the surrogate path more stable, easier to follow, and less sensitive to stochastic fluctuations. In this sense, Bézier surrogates are not just approximations of SGD trajectories, but improved optimisation paths that preserve essential functionality while reducing geometric complexity. Theorem \ref{thm:bez-sur} formalises these guarantees.

\begin{theorem}
\label{thm:bez-sur}
Let $\mathcal{L}:\Theta \to \mathbb{R}$ be a $\beta$-smooth, lower-bounded loss function, $\boldsymbol{\theta}_0 \in \Theta$ be a random initialisation with loss $\ell_0 = \mathcal{L}(\boldsymbol{\theta}_0)$, and $\boldsymbol{\theta}_T$ be an SGD endpoint after $K$ steps such that
$\|\nabla \mathcal{L}(\boldsymbol{\theta}_T)\| \le \varepsilon$, where $\mathcal{L}(\boldsymbol{\theta}_T) = \ell_T \ll \ell_0.$ 

\noindent Let $\gamma(t)$ denote the piecewise-linear interpolation of the SGD iterates $\{\boldsymbol{\theta}_k\}_{k=0}^K$. Also, let $\Phi_{\boldsymbol{\phi}}(t)$ denote the quadratic Bézier curve with control point $\boldsymbol{\phi} \in \Theta$ as defined in Eq.~\eqref{eq:bz}. Define the optimised Bézier path and its curvature $\kappa$ as
\begin{equation*}
    \Phi_{\boldsymbol{\phi}^\star}(t) := \Phi_{\boldsymbol{\phi}}(t) \quad \text{with} \quad 
    \boldsymbol{\phi}^\star = \arg\min_{\boldsymbol{\phi}} \int_0^1 \mathcal{L}(\Phi_{\boldsymbol{\phi}}(t))\, dt,
\end{equation*}
\begin{equation*}
    \kappa := 2\|\boldsymbol{\theta}_0 - 2\boldsymbol{\phi}^\star + \boldsymbol{\theta}_T\|.
\end{equation*}
\noindent Assume the model map $f_{\boldsymbol{\theta}}(\boldsymbol{x})$ is $L_f$-Lipschitz in $\boldsymbol{\theta}$ for every $\boldsymbol{x} \in \mathcal{X}$. Then:

\begin{enumerate}
\itemsep 0.2em
    \item[\textbf{(i)}] \textbf{Average loss along the Bézier path is near-optimal:}
    \begin{equation}
    \label{eq:bez_loss}
        \int_0^1 \mathcal{L}(\Phi_{\boldsymbol{\phi}^\star}(t))\, dt 
        \le
        \int_0^1 \mathcal{L}(\gamma(t))\, dt + \frac{\beta \kappa^2}{240}.
    \end{equation}
    
    \item[\textbf{(ii)}] \textbf{Bézier path has lower and noise-free curvature compared to SGD trajectory:}
    \begin{equation}
    \label{eq:bez_noise}
        \sup_{t \in [0, 1]} \|\Phi_{\boldsymbol{\phi}^\star}''(t)\| = \kappa, \quad
        \mathbb{E}\left[\sup_t \|\gamma''(t)\|\right] \ge \kappa + c \cdot \sigma_{\mathrm{sgd}},
    \end{equation}
    where $\sigma_{\mathrm{sgd}}^2$ is the variance of stochastic gradient noise, and $c > 0$ is a constant dependent on the step size.

    \item[\textbf{(iii)}] \textbf{Model predictions along Bézier path remain close to those along SGD:}
    \begin{equation}
    \label{eq:bez_pred}
        \sup_{\boldsymbol{x} \in \mathcal{X},\; t \in [0, 1]} 
        \|f_{\Phi_{\boldsymbol{\phi}^\star}(t)}(\boldsymbol{x}) - f_{\gamma(t)}(\boldsymbol{x})\| \le \frac{L_f \kappa}{8}.
    \end{equation}

\end{enumerate}

\noindent Hence, the optimised Bézier path $\Phi_{\boldsymbol{\phi}^\star}(t)$ serves as a smooth, low-loss, and functionally faithful surrogate to the original SGD trajectory.
\end{theorem}

\begin{proof}
    Section~\ref{sec:thm1-proof} of the supplementary material provides the complete proof.
\end{proof}

These guarantees establish Bézier surrogates as smooth, low-loss, and functionally faithful replacements for noisy SGD trajectories. Crucially, their reduced curvature and noise-free geometry make them more stable and efficient to use within trajectory-based optimisation. Building on this foundation, we now propose a complete condensation framework that leverages these surrogates to guide synthetic datasets in place of raw SGD trajectories.

\subsection{Dataset Condensation Using Surrogates}
The proposed DC framework, termed Bézier Trajectory Matching (BTM), begins with a randomly initialised synthetic dataset that is iteratively optimised using Bézier surrogates. It assumes access to a collection of $K$ optimised Bézier surrogates ${\Phi_{\boldsymbol{\phi}^\star}^{(k)}}_{k=1}^K$ that approximate SGD trajectories on the real dataset and a synthetic dataset $\tilde{\mathcal{D}} = \{(\boldsymbol{x}_i, \boldsymbol{y}_i)\}_{i=1}^{|\tilde{\mathcal{D}}|}$, where we stack the synthetic inputs into a matrix $\tilde{\mathbf{X}} \in \mathbb{R}^{{|\tilde{\mathcal{D}}|} \times d}$  and the associated labels into a tensor $\tilde{\mathbf{Y}} \in \mathbb{R}^{{|\tilde{\mathcal{D}}|} \times C}$.

At each iteration, a surrogate $\Phi_{\boldsymbol{\phi}^\star}^{(k)}$ is sampled. 
Two interpolation parameters $t_{\text{start}}, t_{\text{end}} \sim \mathcal{U}(0,1)$ 
are drawn such that $t_{\text{start}} < t_{\text{end}}$. 
These define the segment of the surrogate path used for matching: $\boldsymbol{\theta}_{\text{start}} = \Phi_{\boldsymbol{\phi}^\star}^{(k)}(t_{\text{start}})$ and  $\boldsymbol{\theta}_{\text{target}} = \Phi_{\boldsymbol{\phi}^\star}^{(k)}(t_{\text{end}})$. A student model is initialised at $\tilde{\boldsymbol{\theta}}_0 = \boldsymbol{\theta}_{\text{start}}$, and is trained for $N$ gradient steps on mini-batches 
$B_i = \{(\boldsymbol{x}_j, \boldsymbol{y}_j)\}_{j=1}^b$ 
of size $b$ sampled from the synthetic dataset $\tilde{\mathcal{D}}$ as: 
\begin{equation}
\tilde{\boldsymbol{\theta}}_{i+1} =
\tilde{\boldsymbol{\theta}}_i - \eta_s 
\nabla_{\tilde{\boldsymbol{\theta}}_i} \!\left[
  \frac{1}{b} \sum_{\substack{(\boldsymbol{x},\boldsymbol{y}) \\ \in B_i}}
  \ell(f_{\tilde{\boldsymbol{\theta}}_i}(\boldsymbol{x}), \boldsymbol{y})
\right],
\end{equation}

where $\ell(\cdot,\cdot)$ is the supervised loss and $\eta_s$ is the student learning rate. 
After $N$ steps, the student reaches $\tilde{\boldsymbol{\theta}}_N$.  

To align the student’s progression with the Bézier surrogate, we measure the discrepancy 
using a normalised matching loss:
\begin{equation}
\mathcal{L}_{\text{BTM}} = 
\frac{\|\tilde{\boldsymbol{\theta}}_N - \boldsymbol{\theta}_{\text{target}}\|_2^2}
     {\|\boldsymbol{\theta}_{\text{start}} - \boldsymbol{\theta}_{\text{target}}\|_2^2}.
\end{equation}

The synthetic data $\tilde{\mathbf{X}}$ is updated to minimise $\mathcal{L}_{\text{BTM}}$. Since this loss depends on the synthetic data only indirectly through the unrolled student updates, exact differentiation would 
require second-order derivatives through the full trajectory. To avoid this, we 
adopt a first-order approximation and use the gradient of the BTM loss with respect 
to the final student parameters as the update signal:
\begin{equation}
\boldsymbol{g}_L = 
\nabla_{\tilde{\boldsymbol{\theta}}_N} \mathcal{L}_{\text{BTM}}
= \frac{2(\tilde{\boldsymbol{\theta}}_N - \boldsymbol{\theta}_{\text{target}})}
       {\|\boldsymbol{\theta}_{\text{start}} - \boldsymbol{\theta}_{\text{target}}\|_2^2}.
\end{equation}

The corresponding approximate meta-gradient with respect to the synthetic inputs is
\begin{align}
    \nabla_{\tilde{\mathbf{X}}}\mathcal{L}_{\text{BTM}} \;\approx\;& \\[0.5ex]
    &\hspace{-3em} -\eta_s \sum_{i=0}^{N-1} \frac{1}{|B_i|}
       \sum_{\substack{(\boldsymbol{x},\boldsymbol{y}) \\ \in B_i}}
       \nabla_{\tilde{\mathbf{X}}}\,\Big\langle
       \nabla_{\tilde{\boldsymbol{\theta}}_i}
       \ell\!\big(f_{\tilde{\boldsymbol{\theta}}_i}(\boldsymbol{x}), \boldsymbol{y}\big),\,
       \boldsymbol{g}_L \Big\rangle.
\end{align}

Finally, the synthetic dataset is updated by gradient descent:
\begin{equation}
\tilde{\mathbf{X}} \;\leftarrow\; 
\tilde{\mathbf{X}} - \eta_x\,\nabla_{\tilde{\mathbf{X}}}\mathcal{L}_{\text{BTM}}.
\end{equation}
Note that we can also train labels $\tilde{\mathbf{Y}}$ in the same manner. The complete dataset condensation process is depicted as Algorithm \ref{alg:condensation} in Section \ref{sec:btm-algo} of the supplementary material.

\begin{table*}[t]
\centering
\caption{Performance on CURIAL datasets across different \emph{ipc} levels. Best results at each \emph{ipc} are highlighted in \textcolor{blue}{blue} (ours) and \textcolor{red}{red} (baseline).}

\label{tab:curial_results}

\textbf{(a) Oxford University Hospitals (OUH)}\\[0.5ex]
\begin{sc}
\resizebox{0.85\linewidth}{!}{
    \begin{tabular}{lccccBcccc}
    \toprule
    & \multicolumn{4}{cB}{\textbf{AUROC}} & \multicolumn{4}{c}{\textbf{AUPRC}} \\
    \cmidrule(lr{0.5em}){2-5} \cmidrule(lr{0.5em}){6-9}
    \textbf{Method} & \textbf{50} & \textbf{100} & \textbf{200} & \textbf{500} & 
    \textbf{50} & \textbf{100} & \textbf{200} & \textbf{500} \\
    \midrule
    Random     & $0.835_{\pm 0.021}$ & $0.855_{\pm 0.007}$ & $0.869_{\pm 0.006}$ & $0.888_{\pm 0.004}$ 
               & $0.158_{\pm 0.036}$ & $0.176_{\pm 0.029}$ & $0.219_{\pm 0.024}$ & $0.276_{\pm 0.031}$ \\
    M3D        & $0.840_{\pm 0.004}$ & $0.862_{\pm 0.003}$ & $0.872_{\pm 0.003}$ & \bestOther{$0.893_{\pm 0.002}$} 
               & $0.162_{\pm 0.015}$ & $0.190_{\pm 0.010}$ & $0.266_{\pm 0.010}$ & $0.290_{\pm 0.012}$ \\
    MTT        & $0.824_{\pm 0.016}$ & $0.849_{\pm 0.011}$ & $0.855_{\pm 0.008}$ & $0.870_{\pm 0.005}$ 
               & $0.356_{\pm 0.019}$ & $0.381_{\pm 0.019}$ & $0.405_{\pm 0.012}$ & $0.407_{\pm 0.007}$ \\
    TESLA      & $0.839_{\pm 0.007}$ & $0.875_{\pm 0.002}$ & $0.874_{\pm 0.003}$ & $0.880_{\pm 0.002}$ 
               & $0.169_{\pm 0.012}$ & $0.205_{\pm 0.009}$ & $0.202_{\pm 0.010}$ & $0.217_{\pm 0.011}$ \\
    FTD        & $0.831_{\pm 0.014}$ & $0.847_{\pm 0.006}$ & $0.852_{\pm 0.005}$ & $0.860_{\pm 0.006}$ 
               & $0.321_{\pm 0.031}$ & $0.382_{\pm 0.012}$ & $0.400_{\pm 0.009}$ & $0.400_{\pm 0.023}$ \\
    DATM       & $0.829_{\pm 0.010}$ & $0.844_{\pm 0.007}$ & $0.851_{\pm 0.009}$ & $0.872_{\pm 0.009}$ 
               & $0.338_{\pm 0.020}$ & $0.394_{\pm 0.009}$ & $0.414_{\pm 0.009}$ & $0.409_{\pm 0.004}$ \\
    \rowcolor{btmgray}
    BTM (Ours) & \bestOurs{$0.854_{\pm 0.012}$} & \bestOurs{$0.863_{\pm 0.009}$} & \bestOurs{$0.876_{\pm 0.00}$} & $0.888_{\pm 0.003}$ 
           & \bestOurs{$0.394_{\pm 0.014}$} & \bestOurs{$0.396_{\pm 0.010}$} & \bestOurs{$0.427_{\pm 0.006}$} & \bestOurs{$0.436_{\pm 0.003}$} \\
    \midrule
    \emph{\textbf{Full Dataset}} & \multicolumn{4}{cB}{$\mathbf{0.901_{\pm 0.001}}$} & \multicolumn{4}{c}{$\mathbf{0.445_{\pm 0.004}}$} \\
    \bottomrule
    \end{tabular}
    }
\end{sc}
\par\vspace{3ex}

\textbf{(b) Portsmouth University Hospitals (PUH)}\\[0.5ex]
\begin{sc}
\resizebox{0.85\linewidth}{!}{
    \begin{tabular}{lccccBcccc}
    \toprule
    & \multicolumn{4}{cB}{\textbf{AUROC}} & \multicolumn{4}{c}{\textbf{AUPRC}} \\
    \cmidrule(lr{0.5em}){2-5} \cmidrule(lr{0.5em}){6-9}
    \textbf{Method} & \textbf{50} & \textbf{100} & \textbf{200} & \textbf{500} & \textbf{50} & \textbf{100} & \textbf{200} & \textbf{500} \\
    \midrule
    Random     & $0.858_{\pm 0.006}$ & $0.867_{\pm 0.004}$ & $0.879_{\pm 0.004}$ & $0.893_{\pm 0.004}$
               & $0.322_{\pm 0.025}$ & $0.347_{\pm 0.011}$ & $0.410_{\pm 0.029}$ & $0.473_{\pm 0.018}$ \\
    M3D        & $0.863_{\pm 0.004}$ & $0.883_{\pm 0.003}$ & $0.888_{\pm 0.002}$ & $0.900_{\pm 0.001}$
               & $0.348_{\pm 0.013}$ & $0.405_{\pm 0.011}$ & $0.416_{\pm 0.011}$ & $0.527_{\pm 0.011}$ \\
    MTT        & $0.845_{\pm 0.009}$ & $0.871_{\pm 0.005}$ & $0.887_{\pm 0.005}$ & $0.895_{\pm 0.004}$
               & $0.495_{\pm 0.013}$ & $0.517_{\pm 0.012}$ & $0.529_{\pm 0.013}$ & $0.544_{\pm 0.018}$ \\
    TESLA      & $0.861_{\pm 0.006}$ & $0.861_{\pm 0.006}$ & $0.885_{\pm 0.004}$ & $0.893_{\pm 0.002}$
               & $0.412_{\pm 0.019}$ & $0.385_{\pm 0.016}$ & $0.472_{\pm 0.009}$ & $0.483_{\pm 0.009}$ \\
    FTD        & $0.845_{\pm 0.007}$ & $0.870_{\pm 0.007}$ & $0.885_{\pm 0.006}$ & $0.899_{\pm 0.003}$
               & $0.499_{\pm 0.013}$ & $0.536_{\pm 0.010}$ & $0.550_{\pm 0.010}$ & $0.584_{\pm 0.007}$ \\
    DATM       & $0.871_{\pm 0.005}$ & $0.876_{\pm 0.004}$ & $0.887_{\pm 0.004}$ & $0.889_{\pm 0.003}$
               & $0.540_{\pm 0.007}$ & $0.555_{\pm 0.008}$ & $0.555_{\pm 0.008}$ & $0.579_{\pm 0.005}$ \\
    \rowcolor{btmgray}
    BTM (Ours) & \bestOurs{$0.881_{\pm 0.007}$} & \bestOurs{$0.895_{\pm 0.007}$} & \bestOurs{$0.902_{\pm 0.003}$} & \bestOurs{$0.905_{\pm 0.002}$}
           & \bestOurs{$0.564_{\pm 0.011}$} & \bestOurs{$0.575_{\pm 0.008}$} & \bestOurs{$0.593_{\pm 0.009}$} & \bestOurs{$0.603_{\pm 0.007}$} \\
    \midrule
    \emph{\textbf{Full Dataset}} & \multicolumn{4}{cB}{$\mathbf{0.906_{\pm 0.002}}$} & \multicolumn{4}{c}{$\mathbf{0.610_{\pm 0.004}}$} \\
    \bottomrule
    \end{tabular}
    }
\end{sc}
\par\vspace{3ex}

\textbf{(c) University Hospitals Birmingham (UHB)}\\[0.5ex]
\begin{sc}
\resizebox{0.85\linewidth}{!}{
    \begin{tabular}{lccccBcccc}
    \toprule
    & \multicolumn{4}{cB}{\textbf{AUROC}} & \multicolumn{4}{c}{\textbf{AUPRC}} \\
    \cmidrule(lr{0.5em}){2-5} \cmidrule(lr{0.5em}){6-9}
    \textbf{Method} & \textbf{50} & \textbf{100} & \textbf{200} & \textbf{500} & \textbf{50} & \textbf{100} & \textbf{200} & \textbf{500} \\
    \midrule  
    Random      & $0.847_{\pm 0.015}$ & \bestOther{$0.856_{\pm 0.010}$} & $0.876_{\pm 0.007}$ & $0.891_{\pm 0.005}$
                & $0.089_{\pm 0.024}$ & $0.108_{\pm 0.018}$ & $0.126_{\pm 0.019}$ & $0.153_{\pm 0.015}$ \\
    M3D         & \bestOther{$0.863_{\pm 0.005}$} & $0.853_{\pm 0.003}$ & $0.872_{\pm 0.004}$ & \bestOther{$0.891_{\pm 0.003}$}
                & $0.107_{\pm 0.005}$ & $0.070_{\pm 0.006}$ & $0.107_{\pm 0.008}$ & $0.153_{\pm 0.006}$ \\
    MTT         & $0.802_{\pm 0.008}$ & $0.847_{\pm 0.013}$ & $0.871_{\pm 0.013}$ & $0.884_{\pm 0.007}$
                & $0.092_{\pm 0.065}$ & $0.228_{\pm 0.021}$ & $0.242_{\pm 0.016}$ & $0.234_{\pm 0.017}$ \\
    TESLA       & $0.820_{\pm 0.001}$ & $0.850_{\pm 0.004}$ & $0.859_{\pm 0.004}$ & $0.873_{\pm 0.006}$
                & $0.099_{\pm 0.001}$ & $0.079_{\pm 0.006}$ & $0.125_{\pm 0.008}$ & $0.131_{\pm 0.014}$ \\
    FTD         & $0.830_{\pm 0.005}$ & $0.839_{\pm 0.010}$ & $0.847_{\pm 0.010}$ & $0.872_{\pm 0.010}$
                & $0.122_{\pm 0.056}$ & $0.218_{\pm 0.013}$ & $0.216_{\pm 0.018}$ & $0.233_{\pm 0.012}$ \\
    DATM        & $0.828_{\pm 0.028}$ & $0.832_{\pm 0.022}$ & $0.843_{\pm 0.015}$ & $0.870_{\pm 0.005}$
                & $0.152_{\pm 0.019}$ & $0.176_{\pm 0.022}$ & $0.221_{\pm 0.017}$ & $0.237_{\pm 0.017}$ \\
    \rowcolor{btmgray}
    BTM (Ours) & $0.842_{\pm 0.011}$ & $0.836_{\pm 0.023}$ & \bestOurs{$0.884_{\pm 0.007}$} & $0.891_{\pm 0.006}$
            & \bestOurs{$0.258_{\pm 0.006}$} & \bestOurs{$0.246_{\pm 0.013}$} & \bestOurs{$0.253_{\pm 0.011}$} & \bestOurs{$0.272_{\pm 0.006}$} \\
    \midrule
    \emph{\textbf{Full Dataset}} & \multicolumn{4}{cB}{$\mathbf{0.895_{\pm 0.005}}$} & \multicolumn{4}{c}{$\mathbf{0.284_{\pm 0.005}}$} \\
    \bottomrule
    \end{tabular}
    }
\end{sc}

\end{table*}

\section{EXPERIMENTAL SETUP}
\label{sec:exp}
\textsc{DATASETS}: The proposed method is evaluated on three de-identified clinical datasets spanning tabular and time-series modalities, each framed as binary classification:
\begin{itemize}[noitemsep, topsep=0pt]

    \item \textsc{CURIAL} \citep{soltan2024scalable}: Anonymised EHR data from emergency departments at three NHS Trusts in the UK—Oxford (161,955 examples), Portsmouth (38,717), and Birmingham (95,236)—comprising demographics, blood tests, and vital signs, used for COVID-19 diagnosis.
    
    \item \textsc{eICU} \citep{pollard2018eicu}: 49,305 multi-centre ICU stays represented by 402 tabular features (admission, ICU, and demographic variables), used for in-hospital mortality prediction.  

    \item \textsc{MIMIC-III} \citep{johnson2016mimic}: This large critical care dataset is processed for in-hospital mortality prediction. We use 21,156 ICU stays represented as multivariate time-series with 48 time-steps and 60 physiological features.
\end{itemize}

Full dataset descriptions and preprocessing details are provided in Section \ref{sec:dataset_details} of the supplementary material.

\textsc{BASELINES}: The proposed approach is compared against state-of-the-art dataset condensation methods spanning trajectory matching (MTT \citep{cazenavette2022dataset}, TESLA \citep{liu2023tesla}, FTD \citep{du2023ftd}, DATM \citep{guo2024datm}) and distribution matching (M3D \citep{zhang2024m3d}). Random selection and full dataset training serve as lower and upper bounds respectively.

\textsc{EXPERIMENTAL PROTOCOL}: Synthetic datasets are condensed at 50, 100, 200, and 500 instances per class (\emph{ipc}). All methods are adapted to our clinical data modalities following their original implementations, with unoptimised labels throughout—hard labels for most methods, soft labels for TESLA. Following standard evaluation protocol, randomly initialised networks are trained from scratch on the condensed data and evaluated on held-out test sets. Performance is measured using area under the receiver operating characteristic curve (AUROC) and area under the precision-recall curve (AUPRC), with results reported as mean $\pm$ standard deviation over 10 random initialisations.

Evaluation architectures remain fixed across all methods: a shallow multi-layer perceptron (MLP) \citep{rumelhart1986learning} for tabular data and a temporal convolutional network (TCN) \citep{bai2018empirical} for time-series. Consistent with prior work, the same architectures are employed for both condensation and evaluation (apart from ablation studies). Synthetic inputs are initialised from real examples unless otherwise specified. Complete implementation details, including parameter settings, are provided in Section~\ref{sec:implementation} of the supplementary material.

\section{RESULTS \& DISCUSSIONS}
\label{sec:results}

\begin{figure}[t]
    \centering
    \captionsetup{skip=1pt}
    \begin{subfigure}[t]{1.0\linewidth}
        \centering
        \includegraphics[width=\linewidth]{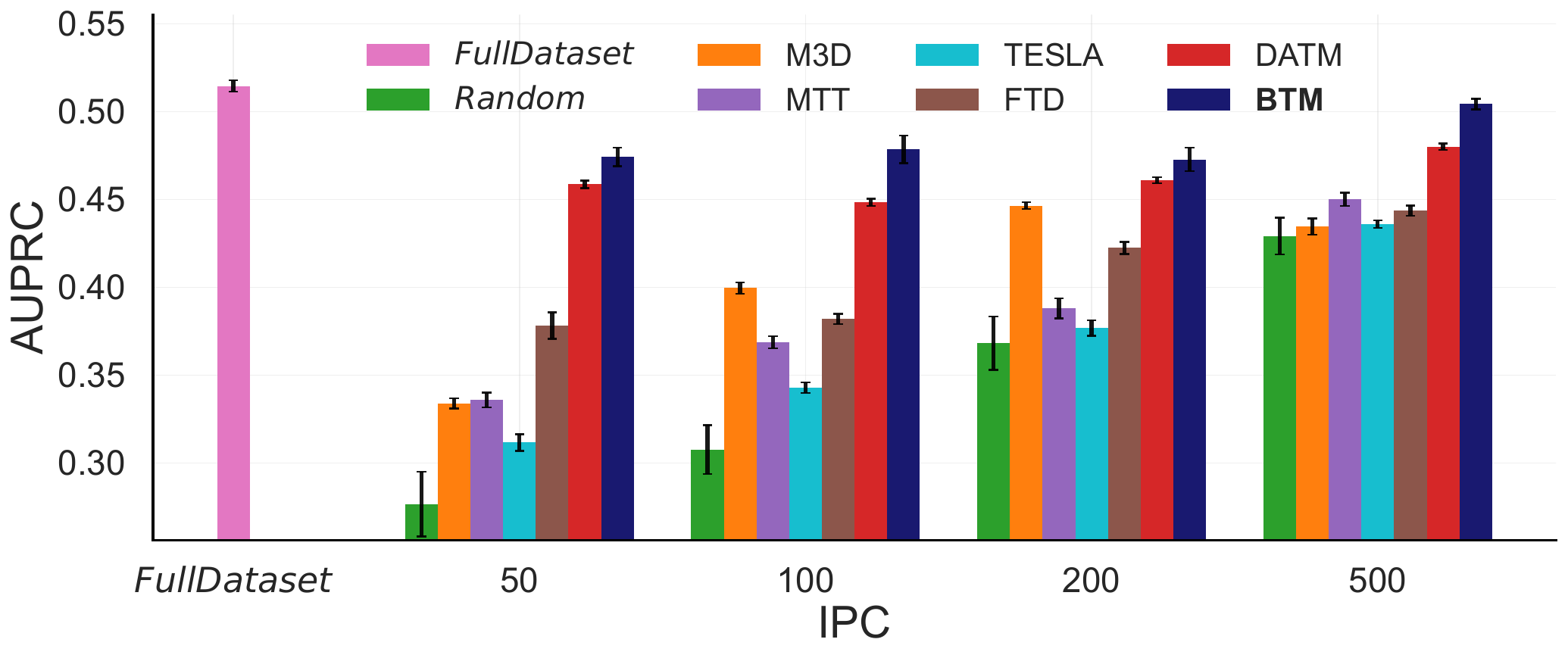}
        \subcaption{eICU}
    \end{subfigure}
    \hfill
    \begin{subfigure}[t]{1.0\linewidth}
        \centering
        \includegraphics[width=\linewidth]{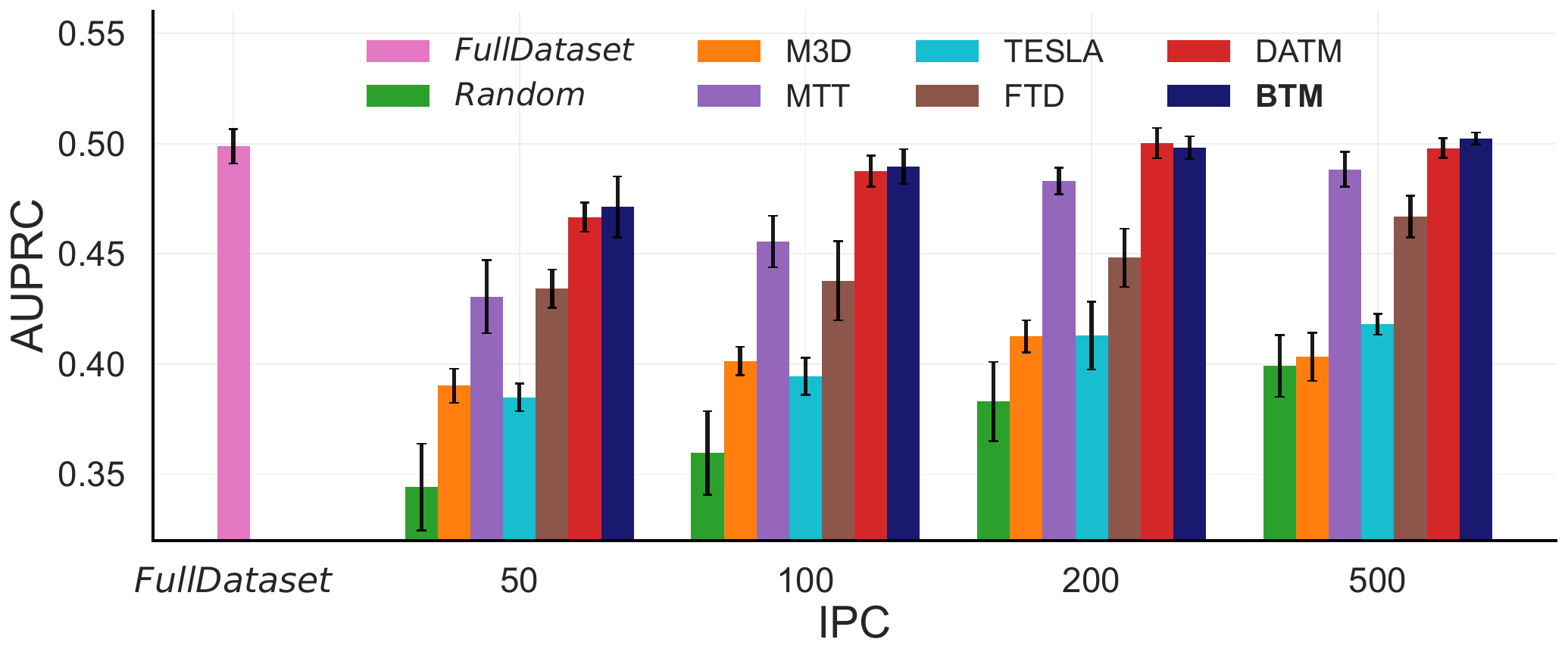}
        \subcaption{MIMIC-III}
    \end{subfigure}
    \caption{
  Performance comparison of different methods across \emph{ipc} levels on eICU and MIMIC-III datasets.
    }
    \label{fig:ihm_results}
\end{figure}

\subsection{Performance of Condensed Catasets}

\textsc{CURIAL Datasets:}
BTM achieves consistently strong results across all three CURIAL sites (Table~\ref{tab:curial_results}), with performance improving as \emph{ipc} increases. Gains are particularly clear in AUPRC, where BTM outperforms all baselines at every site, while AUROC improvements are competitive and generally comparable across methods. At Oxford (OUH) and Portsmouth (PUH), BTM provides the best performance across both metrics, and at Birmingham (UHB) it delivers the largest AUPRC gains despite similar AUROC to strong baselines. In imbalanced prediction tasks, AUPRC provides a more informative evaluation than AUROC, as it directly reflects precision at clinically relevant recall levels. This is particularly important in CURIAL, where outcome prevalence is low—0.8\% at UHB, 1.7\% at OUH, and 5.3\% at PUH. Under these conditions, baselines such as M3D can achieve competitive AUROC yet still perform poorly in detecting positives, as reflected in their lower AUPRC. By contrast, BTM not only delivers consistent gains in AUPRC across all sites but also improves or maintains AUROC relative to strong baselines, highlighting its robustness for imbalanced clinical prediction tasks.

\textsc{eICU and MIMIC-III Datasets:}
On in-hospital mortality tasks (Figure~\ref{fig:ihm_results}), BTM achieves strong AUPRC gains on eICU and is effectively lossless on MIMIC-III at 500 \emph{ipc}. FTD performs poorly relative to MTT, highlighting that enforcing trajectory flatness is less effective than using smoother surrogates. Unlike DATM, BTM achieves comparable performance without requiring \emph{ipc}-specific difficulty alignment, offering a simpler alternative. Complete results including AUROC are in Section \ref{sec:add_res} of the supplementary material.


\textsc{Synthetic Data Compression Ratio:}  
On CURIAL, condensing to 500 \emph{ipc} reduces datasets to between 0.6\% (OUH) and 2.6\% (PUH) of their original size while achieving performance close to the full dataset. On eICU, the synthetic dataset is only 2\% of the original (1,000 vs.\ 49,305) yet maintains near--full--dataset performance. On MIMIC-III, comparable performance is achieved already at 200 \emph{ipc}, where the synthetic dataset is again about 2\% of the original (400 vs.\ 21,156). Across all three settings, BTM demonstrates that extreme compression can be achieved without sacrificing noticeable predictive performance.

\begin{table}[t]
\centering
\caption{Cross-architecture generalisation at 500 \emph{ipc}. Synthetic data condensed with one architecture (grey) is evaluated on unseen architectures: (a) eICU with MLP depth/width variants, (b) MIMIC-III with TCN/LSTM transfer. See Section~\ref{subsec:architectures} of the supplementary material for architecture details.}
\label{tab:crossarch}

\begin{minipage}{0.48\linewidth}
\centering
\textbf{(a) eICU}\\[0.5ex]
\begin{small}
\begin{sc}
\resizebox{\linewidth}{!}{
    \begin{tabular}{lcc}
    \toprule
    \textbf{Network} & \textbf{AUROC} & \textbf{AUPRC}\\
    \midrule
    MLP-1      & $0.871_{\pm0.001}$ & $0.504_{\pm0.001}$ \\
    MLP-2      & $0.873_{\pm0.002}$ & $0.506_{\pm0.002}$ \\
    MLP-3      & $0.871_{\pm0.002}$ & $0.501_{\pm0.001}$ \\
    MLP-4      & $0.867_{\pm0.002}$ & $0.488_{\pm0.001}$ \\
    \rowcolor{btmgray}
    \emph{MLP} & $\mathbf{0.874}_{\pm0.002}$ & $\mathbf{0.504}_{\pm0.001}$ \\
    \bottomrule
    \end{tabular}
}
\end{sc}
\end{small}
\end{minipage}\hfill
\begin{minipage}{0.48\linewidth}
\centering
\textbf{(b) MIMIC-III}\\[0.5ex]
\begin{small}
\begin{sc}
\resizebox{\linewidth}{!}{
    \begin{tabular}{lcc}
    \toprule
    \textbf{Network} & \textbf{AUROC} & \textbf{AUPRC}\\
    \midrule
    LSTM-1     & $0.810_{\pm0.003}$ & $0.440_{\pm0.006}$ \\
    LSTM-2     & $0.819_{\pm0.004}$ & $0.450_{\pm0.008}$ \\
    TCN-1      & $0.835_{\pm0.002}$ & $0.491_{\pm0.005}$ \\
    TCN-2      & $0.831_{\pm0.002}$ & $0.487_{\pm0.003}$ \\
    \rowcolor{btmgray}
    \emph{TCN} & $\mathbf{0.840}_{\pm0.001}$ & $\mathbf{0.502}_{\pm0.003}$ \\
    \bottomrule
    \end{tabular}
}
\end{sc}
\end{small}
\end{minipage}

\end{table}

\subsection{Cross-Architecture Generalisation} 
A key question for DC is whether synthetic data condensed with one model architecture can generalise to others, avoiding the need for repeated re-condensation. Table~\ref{tab:crossarch} evaluates this transferability at 500 \emph{ipc}. On eICU, synthetic data condensed with a 1-layer MLP transfers robustly across MLP variants of different depths and widths, with minimal variation in AUROC and AUPRC. On MIMIC-III, synthetic data condensed with a TCN generalises well to other TCNs and transfers competitively to LSTMs despite their fundamentally different temporal processing mechanisms, showing only modest degradation. Performance remains stable both within and across architecture families, demonstrating that BTM produces architecture-agnostic synthetic data suitable for deployment without re-condensation.

\subsection{Storage Efficiency}
Figure~\ref{fig:storage_cost} highlights the storage efficiency of BTM across all clinical datasets. Reduction factors depend on trajectory length: storing 100 checkpoints per trajectory yields approximately 33$\times$ savings for CURIAL and eICU, while 60 checkpoints yield about 20$\times$ savings for MIMIC-III. These reductions substantially lower memory demands, allowing more expert trajectories to be retained and enabling DC in resource-constrained clinical settings where trajectory storage would otherwise be prohibitive.

\begin{figure}
\centering
    \includegraphics[width=.9\linewidth]{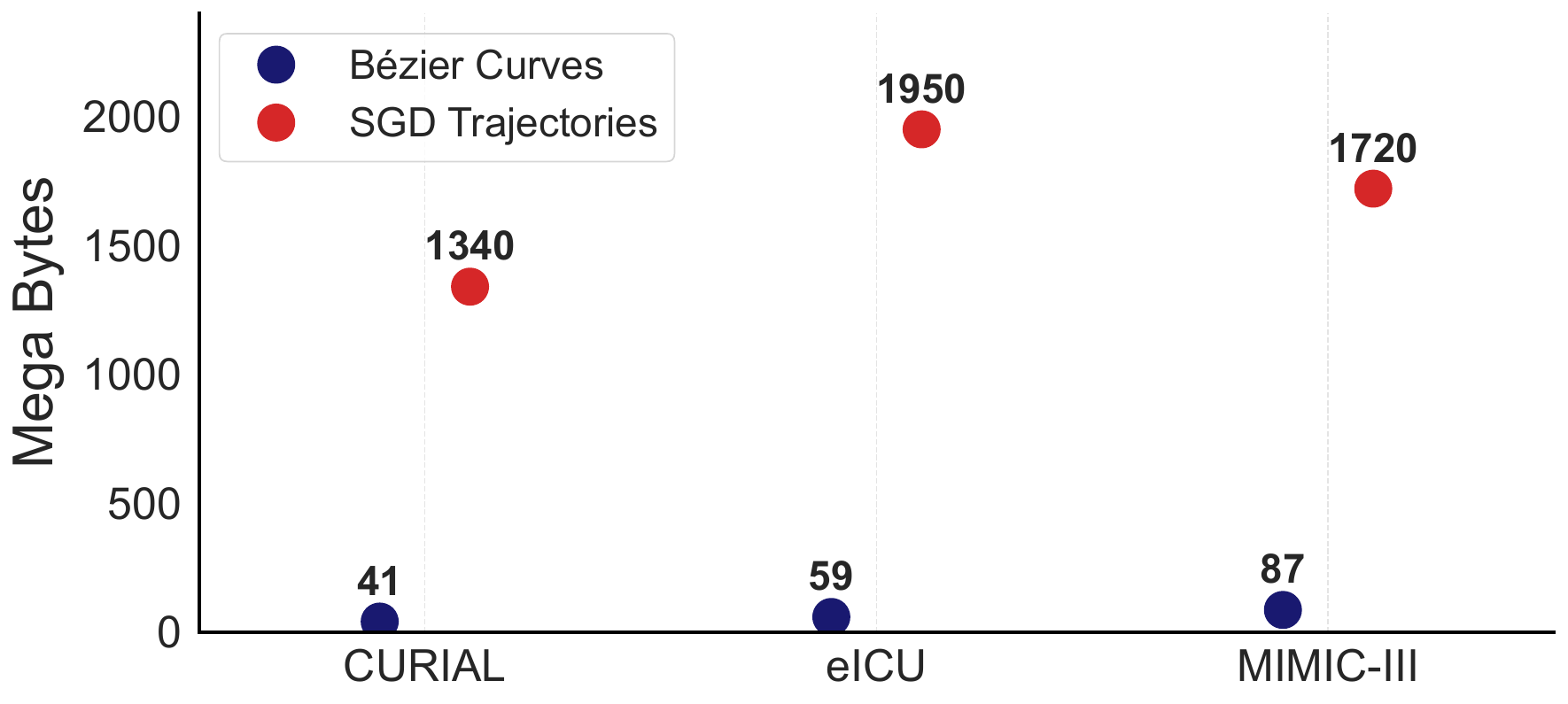}
    \caption{Trajectory storage requirements across clinical datasets. SGD trajectories require 33$\times$ (eICU, CURIAL) and 20$\times$ (MIMIC-III) more storage than Bézier surrogates, which store only initial, final, and control checkpoints.}
\label{fig:storage_cost}
\end{figure}

\subsection{Ablation Studies}

\textsc{Inner Loop Steps:} The inner loop controls the number of gradient updates the student model takes when trained on the synthetic data before matching against the expert trajectory. 

\begin{figure}[H]
\centering
    \includegraphics[width=1\linewidth]{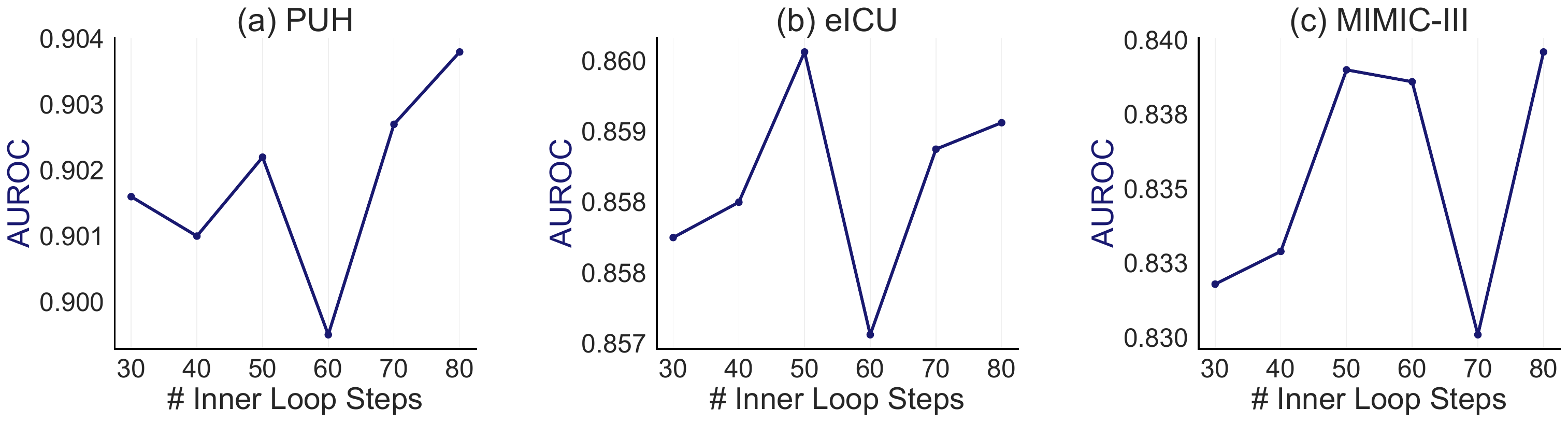}
    \caption{Impact of inner-loop steps $N$ on AUROC performance at 200 \emph{ipc}. BTM achieves strong performance with only 30 steps, reducing computational overhead. Similar trends observed for AUPRC.}
\label{fig:inner_loop}
\end{figure}

In standard TM, this is tied to the expert optimisation epochs M, but with Bézier surrogates the continuous parameterisation $t \in [0,1]$ decouples segment length ($t_{\text{start}}$, $t_{\text{end}}$) from discrete optimisation steps. As a result, the optimal number of steps must be determined empirically. Figure \ref{fig:inner_loop} ablates $N$ at 200 \emph{ipc}. While prior methods typically require $N \ge 40$ \citep{guo2024datm}, BTM sustains strong AUROC and AUPRC even at $N = 30$, demonstrating that Bézier surrogates provide stable supervision, enabling the student to train effectively with few gradient steps.

\textsc{Initialisation:} Table~\ref{tab:init} compares real and random initialisation on eICU. Real initialisation seeds synthetic inputs from real samples; random initialisation samples from class-conditional Gaussians. Performance is comparable, with real initialisation providing modest improvements at 500 \emph{ipc}. Critically, random initialisation maintains competitive performance while eliminating direct dependence on real samples, offering stronger privacy guarantees. Combined with differential privacy \citep{dwork2006calibrating}, this enables secure sharing of condensed clinical datasets to democratise healthcare AI development without compromising patient confidentiality.

\begin{table}
\centering
\caption{Initialisation strategy comparison on eICU. Random initialisation remains competitive while providing stronger privacy guarantees.}
\label{tab:init}
\begin{sc}
\resizebox{1.0\linewidth}{!}{
    \begin{tabular}{ll|cccc}
    \toprule
    & \multicolumn{1}{c}{\multirow{2}{*}{\textbf{Init.}}} & \multicolumn{4}{c}{\textbf{IPC}} \\
    \cmidrule(lr){3-6}
    & \multicolumn{1}{c}{} & \textbf{50} & \textbf{100} & \textbf{200} & \textbf{500} \\
    \midrule
    \multirow{2}{*}{\textbf{AUROC}}
        & Real   & $0.866_{\pm0.001}$ & $0.856_{\pm0.004}$ & $0.858_{\pm0.003}$ & $0.871_{\pm0.001}$ \\
        & Random & $0.852_{\pm0.009}$ & $0.858_{\pm0.004}$ & $0.853_{\pm0.004}$ & $0.862_{\pm0.002}$ \\
    \cmidrule(lr){2-6}
    \multirow{2}{*}{\textbf{AUPRC}}
        & Real   & $0.474_{\pm0.005}$ & $0.479_{\pm0.008}$ & $0.473_{\pm0.007}$ & $0.504_{\pm0.003}$ \\
        & Random & $0.463_{\pm0.015}$ & $0.474_{\pm0.007}$ & $0.475_{\pm0.008}$ & $0.499_{\pm0.004}$ \\
    \bottomrule
    \end{tabular}
}
\end{sc}
\end{table}

\section{CONCLUSION}


This work introduces mode connectivity–based trajectory surrogates to overcome key limitations of trajectory matching in dataset condensation. By replacing unstable, storage-intensive SGD trajectories with smooth Bézier paths, we generate condensed clinical datasets that enable effective model development while requiring only three checkpoints per trajectory. Our theoretical analysis shows optimised Bézier curves preserve SGD functional properties while reducing curvature and eliminating stochastic noise. Empirically, across five clinical datasets, our approach consistently matches or outperforms state-of-the-art methods, with particularly strong AUPRC gains for imbalanced tasks.

The framework achieves near–full-dataset performance with up to 99\% compression and order-of-magnitude storage reductions. Cross-architecture experiments confirm transferability without re-condensation, and random initialisation enhances privacy guarantees. A current limitation is that we have not incorporated differential privacy into our framework, though this can be readily integrated to provide formal privacy guarantees on the condensed datasets Overall, this framework advances the practical viability of dataset condensation in healthcare, enabling privacy-preserving model development without sacrificing performance.




\setlength{\itemindent}{-\leftmargin}
\makeatletter
\renewcommand{\@biblabel}[1]{}
\makeatother

\bibliography{ref}

\begin{thebibliography}{}

\bibitem[Bai et~al., 2018]{bai2018empirical}
Bai, S., Kolter, J.~Z., and Koltun, V. (2018).
\newblock An empirical evaluation of generic convolutional and recurrent networks for sequence modeling.
\newblock {\em arXiv preprint arXiv:1803.01271}.

\bibitem[Bohdal et~al., 2020]{bohdal2020flexible}
Bohdal, O., Yang, Y., and Hospedales, T. (2020).
\newblock Flexible dataset distillation: Learn labels instead of images.
\newblock {\em arXiv preprint arXiv:2006.08572}.

\bibitem[Carlini et~al., 2022]{carlini2022no}
Carlini, N., Feldman, V., and Nasr, M. (2022).
\newblock No free lunch in ``privacy for free: How does dataset condensation help privacy''.
\newblock {\em arXiv preprint arXiv:2209.14987}.

\bibitem[Cazenavette et~al., 2022]{cazenavette2022dataset}
Cazenavette, G., Wang, T., Torralba, A., Efros, A.~A., and Zhu, J.-Y. (2022).
\newblock Dataset distillation by matching training trajectories.
\newblock In {\em Proceedings of the IEEE/CVF Conference on Computer Vision and Pattern Recognition}, pages 10718--10727.

\bibitem[Draxler et~al., 2018]{draxler2018essentially}
Draxler, F., Veschgini, K., Salmhofer, M., and Hamprecht, F.~A. (2018).
\newblock Essentially no barriers in neural network energy landscape.
\newblock In {\em International Conference on Machine Learning}, pages 1309--1318. PMLR.

\bibitem[Du et~al., 2023]{du2023ftd}
Du, J., Jiang, Y., Tan, V.~Y., Zhou, J.~T., and Li, H. (2023).
\newblock Minimizing the accumulated trajectory error to improve dataset distillation.
\newblock {\em arXiv preprint arXiv:2211.11004}.

\bibitem[Dwork et~al., 2006]{dwork2006calibrating}
Dwork, C., McSherry, F., Nissim, K., and Smith, A. (2006).
\newblock Calibrating noise to sensitivity in private data analysis.
\newblock In {\em Theory of Cryptography Conference}, pages 265--284. Springer.

\bibitem[Garipov et~al., 2018]{garipov2018loss}
Garipov, T., Izmailov, P., Podoprikhin, D., Vetrov, D.~P., and Wilson, A.~G. (2018).
\newblock Loss surfaces, mode connectivity, and fast ensembling of dnns.
\newblock In {\em Advances in Neural Information Processing Systems}, volume~31.

\bibitem[Guo et~al., 2024]{guo2024datm}
Guo, Z., Wang, K., Cazenavette, G., Li, H., Zhang, K., and You, Y. (2024).
\newblock Towards lossless dataset distillation via difficulty-aligned trajectory matching.
\newblock In {\em International Conference on Learning Representations}.

\bibitem[Hochreiter and Schmidhuber, 1997]{hochreiter1997long}
Hochreiter, S. and Schmidhuber, J. (1997).
\newblock Long short-term memory.
\newblock {\em Neural Computation}, 9(8):1735--1780.

\bibitem[Izmailov et~al., 2018]{izmailov2018averaging}
Izmailov, P., Podoprikhin, D., Garipov, T., Vetrov, D., and Wilson, A.~G. (2018).
\newblock Averaging weights leads to wider optima and better generalization.
\newblock In {\em Conference on Uncertainty in Artificial Intelligence}, pages 876--885. PMLR.

\bibitem[Jacot et~al., 2018]{jacot2018neural}
Jacot, A., Gabriel, F., and Hongler, C. (2018).
\newblock Neural tangent kernel: Convergence and generalization in neural networks.
\newblock {\em Advances in Neural Information Processing Systems}, 31.

\bibitem[Johnson et~al., 2016]{johnson2016mimic}
Johnson, A. E.~W., Pollard, T.~J., Shen, L., Li-wei, L.~H., Feng, M., Ghassemi, M., Moody, B., Szolovits, P., Celi, A.~L., and Mark, R.~G. (2016).
\newblock Mimic-iii, a freely accessible critical care database.
\newblock {\em Scientific Data}, 3(1):1--9.

\bibitem[Li et~al., 2025]{li2025dataset}
Li, M., Cui, C., Liu, Q., Deng, R., Yao, T., Lionts, M., and Huo, Y. (2025).
\newblock Dataset distillation in medical imaging: A feasibility study.
\newblock In {\em Medical Imaging 2025: Image Perception, Observer Performance, and Technology Assessment}, volume 13409, pages 172--179. SPIE.

\bibitem[Liu et~al., 2023]{liu2023tesla}
Liu, X., Tapia, M., Mallya, A., Davis, L., and Shrivastava, A. (2023).
\newblock Scaling up dataset distillation to imagenet-1k with constant memory.
\newblock {\em arXiv preprint arXiv:2211.10586}.

\bibitem[Moulines and Bach, 2011]{moulines2011}
Moulines, {\'E}. and Bach, F. (2011).
\newblock Non-asymptotic analysis of stochastic approximation algorithms for machine learning.
\newblock In {\em Advances in Neural Information Processing Systems}, volume~24, pages 451--459.

\bibitem[Nguyen et~al., 2021]{nguyen2021dataset}
Nguyen, T., Chen, Z., and Lee, J. (2021).
\newblock Dataset meta-learning from kernel ridge regression.
\newblock In {\em International Conference on Learning Representations}.

\bibitem[Nguyen et~al., 2022]{nguyen2022dataset}
Nguyen, T., Novak, R., Xiao, L., and Lee, J. (2022).
\newblock Dataset distillation with infinitely wide convolutional networks.
\newblock In {\em Advances in Neural Information Processing Systems}, volume~35, pages 5186--5198.

\bibitem[Pollard et~al., 2018]{pollard2018eicu}
Pollard, T.~J., Johnson, A., Raffa, J.~D., Celi, L.~A., Mark, R.~G., and Badawi, O. (2018).
\newblock The eicu collaborative research database, a freely available multi-center database for critical care research.
\newblock {\em Scientific Data}, 5(1):1--13.

\bibitem[Pottmann and Wallner, 2001]{pottmann2001}
Pottmann, H. and Wallner, J. (2001).
\newblock {\em Computational Line Geometry}.
\newblock Springer.

\bibitem[Rajkomar et~al., 2019]{rajkomar2019machine}
Rajkomar, A., Dean, J., and Kohane, I. (2019).
\newblock Machine learning in medicine.
\newblock {\em New England Journal of Medicine}, 380(14):1347--1358.

\bibitem[Rumelhart et~al., 1986]{rumelhart1986learning}
Rumelhart, D.~E., Hinton, G.~E., and Williams, R.~J. (1986).
\newblock Learning representations by back-propagating errors.
\newblock {\em Nature}, 323(6088):533--536.

\bibitem[Shabani, 2019]{shabani2019gdpr}
Shabani, M. (2019).
\newblock The impact of the general data protection regulation (gdpr) on artificial intelligence.
\newblock {\em European Journal of Human Genetics}, 27(1):15--17.

\bibitem[Soltan et~al., 2024]{soltan2024scalable}
Soltan, A.~A., Thakur, A., Yang, J., Chauhan, A., D'Cruz, L.~G., Dickson, P., Soltan, M.~A., Thickett, D.~R., Eyre, D.~W., Zhu, T., et~al. (2024).
\newblock A scalable federated learning solution for secondary care using low-cost microcomputing: privacy-preserving development and evaluation of a covid-19 screening test in uk hospitals.
\newblock {\em The Lancet Digital Health}, 6(2):e93--e104.

\bibitem[Sucholutsky and Schonlau, 2019]{sucholutsky2019soft}
Sucholutsky, I. and Schonlau, M. (2019).
\newblock Soft-label dataset distillation and text dataset distillation.
\newblock {\em arXiv preprint arXiv:1910.02551}.

\bibitem[Sun et~al., 2023]{sun2023rded}
Sun, P., Shi, B., Yu, D., and Lin, T. (2023).
\newblock On the diversity and realism of distilled dataset: An efficient dataset distillation paradigm.
\newblock In {\em Proceedings of the IEEE/CVF Conference on Computer Vision and Pattern Recognition}, pages 3928--3937.

\bibitem[Thakur et~al., 2025]{thakur2025optimising}
Thakur, A., Molaei, S., Schwab, P., Belgrave, D., Branson, K., and Clifton, D.~A. (2025).
\newblock Optimising clinical federated learning through mode connectivity-based model aggregation.
\newblock In {\em International Conference on Artificial Intelligence and Statistics}, pages 163--171. PMLR.

\bibitem[Thakur et~al., 2023]{thakur2023clinical}
Thakur, A., Wang, C., Ceritli, T., Clifton, D., and Eyre, D. (2023).
\newblock Mode connections for clinical incremental learning: Lessons from the covid-19 pandemic.
\newblock {\em medRxiv}.
\newblock ICML Workshop on Interpretable Machine Learning in Healthcare.

\bibitem[Thakur et~al., 2024]{thakur2024data}
Thakur, A., Zhu, T., Abrol, V., Armstrong, J., Wang, Y., and Clifton, D.~A. (2024).
\newblock Data encoding for healthcare data democratization and information leakage prevention.
\newblock {\em Nature Communications}, 15(1):1582.

\bibitem[Topol, 2019]{topol2019high}
Topol, E.~J. (2019).
\newblock High-performance medicine: the convergence of human and artificial intelligence.
\newblock {\em Nature Medicine}, 25(1):44--56.

\bibitem[Wang et~al., 2022]{wang2022cafe}
Wang, K., Zhao, B., Peng, X., Zhu, Z., Yang, S., Wang, S., Huang, G., Bilen, H., Wang, X., and You, Y. (2022).
\newblock Cafe: Learning to condense dataset by aligning features.
\newblock In {\em Proceedings of the IEEE/CVF Conference on Computer Vision and Pattern Recognition}, pages 12196--12205.

\bibitem[Wang et~al., 2018]{wang2018dataset}
Wang, T., Zhu, J.-Y., Torralba, A., and Efros, A.~A. (2018).
\newblock Dataset distillation.
\newblock {\em arXiv preprint arXiv:1811.10959}.

\bibitem[Wang et~al., 2023]{wang2023medical}
Wang, Y., Thakur, A., Dong, M., Ma, P., Petridis, S., Shang, L., Zhu, T., and Clifton, D.~A. (2023).
\newblock Medical records condensation: a roadmap towards healthcare data democratisation.
\newblock {\em arXiv preprint arXiv:2305.03711}.

\bibitem[Yin et~al., 2023]{yin2023sre2l}
Yin, Z., Xing, E., and Shen, Z. (2023).
\newblock Squeeze, recover and relabel: Dataset condensation at imagenet scale from a new perspective.
\newblock In {\em Advances in Neural Information Processing Systems}, volume~36.

\bibitem[Zhang et~al., 2024]{zhang2024m3d}
Zhang, H., Li, S., Wang, P., Zeng, D., and Ge, S. (2024).
\newblock M3d: Dataset condensation by minimizing maximum mean discrepancy.
\newblock {\em arXiv preprint arXiv:2312.15927}.

\bibitem[Zhao and Bilen, 2021]{zhao2021dataset}
Zhao, B. and Bilen, H. (2021).
\newblock Dataset condensation with gradient matching.
\newblock In {\em International Conference on Learning Representations}.

\bibitem[Zhao and Bilen, 2023]{zhao2023dataset}
Zhao, B. and Bilen, H. (2023).
\newblock Dataset condensation with distribution matching.
\newblock In {\em Proceedings of the IEEE/CVF Winter Conference on Applications of Computer Vision}, pages 6514--6523.

\bibitem[Zhao et~al., 2023]{zhao2023improved}
Zhao, G., Li, G., Qin, Y., and Yu, Y. (2023).
\newblock Improved distribution matching for dataset condensation.
\newblock {\em arXiv preprint arXiv:2307.09742}.

\end{thebibliography}

\section*{Checklist}

\begin{enumerate}

  \item For all models and algorithms presented, check if you include:
  \begin{enumerate}
    \item A clear description of the mathematical setting, assumptions, algorithm, and/or model. [Yes]
    \item An analysis of the properties and complexity (time, space, sample size) of any algorithm. [Yes]
    \item (Optional) Anonymized source code, with specification of all dependencies, including external libraries. 
  \end{enumerate}

  \item For any theoretical claim, check if you include:
  \begin{enumerate}
    \item Statements of the full set of assumptions of all theoretical results. [Yes]
    \item Complete proofs of all theoretical results. [Yes]
    \item Clear explanations of any assumptions. [Yes]     
  \end{enumerate}

  \item For all figures and tables that present empirical results, check if you include:
  \begin{enumerate}
    \item The code, data, and instructions needed to reproduce the main experimental results (either in the supplemental material or as a URL). [Yes]
    \item All the training details (e.g., data splits, hyperparameters, how they were chosen). [Yes]
    \item A clear definition of the specific measure or statistics and error bars (e.g., with respect to the random seed after running experiments multiple times). [Yes]
    \item A description of the computing infrastructure used. (e.g., type of GPUs, internal cluster, or cloud provider). [Yes]
  \end{enumerate}

  \item If you are using existing assets (e.g., code, data, models) or curating/releasing new assets, check if you include:
  \begin{enumerate}
    \item Citations of the creator If your work uses existing assets. [Yes]
    \item The license information of the assets, if applicable. [Not Applicable]
    \item New assets either in the supplemental material or as a URL, if applicable. [Yes]
    \item Information about consent from data providers/curators. [Not Applicable]
    \item Discussion of sensible content if applicable, e.g., personally identifiable information or offensive content. [Not Applicable]
  \end{enumerate}

  \item If you used crowdsourcing or conducted research with human subjects, check if you include:
  \begin{enumerate}
    \item The full text of instructions given to participants and screenshots. [Not Applicable]
    \item Descriptions of potential participant risks, with links to Institutional Review Board (IRB) approvals if applicable. [Not Applicable]
    \item The estimated hourly wage paid to participants and the total amount spent on participant compensation. [Not Applicable]
  \end{enumerate}

\end{enumerate}

\clearpage
\appendix
\thispagestyle{empty}

\onecolumn
\aistatstitle{Supplementary Material}

\setcounter{section}{0}
\setcounter{theorem}{0}
\setcounter{equation}{0}
\setcounter{table}{0}
\setcounter{lemma}{0}

\section{PIPELINE OF PROPOSED METHOD}
\label{sec:btm-algo}

\begin{algorithm}[H]
\caption{Dataset Condensation Using Trajectory Surrogates}
\label{alg:condensation}
\begin{algorithmic}[1]
\REQUIRE Collection of $K$ optimised Bézier surrogates $\{\Phi_{\boldsymbol{\phi}^\star}^{(k)}\}_{k=1}^K$, synthetic dataset $\tilde{\mathcal{D}} = \{(\boldsymbol{x}_i, \boldsymbol{y}_i)\}_{i=1}^{|\tilde{\mathcal{D}}|}$, number of classes $C$, input dimension $d$, student learning rate $\eta_s$, meta learning rate $\eta_x$, number of student steps $N$, batch size $b$, maximum iterations $T_{\max}$
\ENSURE Optimised synthetic dataset $\tilde{\mathcal{D}}^\star$
\STATE Stack synthetic inputs into matrix $\tilde{\mathbf{X}} \in \mathbb{R}^{|\tilde{\mathcal{D}}| \times d}$
\STATE Stack synthetic labels into tensor $\tilde{\mathbf{Y}} \in \mathbb{R}^{|\tilde{\mathcal{D}}| \times C}$
\FOR{$t = 1$ to $T_{\max}$}
    \STATE Sample surrogate index $k \sim \mathcal{U}\{1, \ldots, K\}$
    \STATE Sample $t_{\text{start}}, t_{\text{end}} \sim \mathcal{U}(0,1)$ such that $t_{\text{start}} < t_{\text{end}}$
    \STATE $\boldsymbol{\theta}_{\text{start}} \leftarrow \Phi_{\boldsymbol{\phi}^\star}^{(k)}(t_{\text{start}})$ \COMMENT{Start position on surrogate}
    \STATE $\boldsymbol{\theta}_{\text{target}} \leftarrow \Phi_{\boldsymbol{\phi}^\star}^{(k)}(t_{\text{end}})$ \COMMENT{Target position on surrogate}
    \STATE $\tilde{\boldsymbol{\theta}}_0 \leftarrow \boldsymbol{\theta}_{\text{start}}$ \COMMENT{Initialise student model}
    \FOR{$i = 0$ to $N-1$}
        \STATE Sample mini-batch $B_i = \{(\boldsymbol{x}_j, \boldsymbol{y}_j)\}_{j=1}^b \subset \tilde{\mathcal{D}}$
        \STATE $\tilde{\boldsymbol{\theta}}_{i+1} \leftarrow \tilde{\boldsymbol{\theta}}_i - \eta_s \nabla_{\tilde{\boldsymbol{\theta}}_i} \!\left[\frac{1}{b} \sum_{(\boldsymbol{x},\boldsymbol{y}) \in B_i} \ell(f_{\tilde{\boldsymbol{\theta}}_i}(\boldsymbol{x}), \boldsymbol{y})\right]$
    \ENDFOR
    \STATE $\mathcal{L}_{\text{BTM}} \leftarrow \frac{\|\tilde{\boldsymbol{\theta}}_N - \boldsymbol{\theta}_{\text{target}}\|_2^2}{\|\boldsymbol{\theta}_{\text{start}} - \boldsymbol{\theta}_{\text{target}}\|_2^2}$ \COMMENT{Normalised matching loss}
    \STATE $\boldsymbol{g}_L \leftarrow \nabla_{\tilde{\boldsymbol{\theta}}_N} \mathcal{L}_{\text{BTM}} = \frac{2(\tilde{\boldsymbol{\theta}}_N - \boldsymbol{\theta}_{\text{target}})}{\|\boldsymbol{\theta}_{\text{start}} - \boldsymbol{\theta}_{\text{target}}\|_2^2}$ \COMMENT{Gradient signal}
    \STATE Compute $\nabla_{\tilde{\mathbf{X}}}\mathcal{L}_{\text{BTM}} \approx -\eta_s \sum_{i=0}^{N-1} \frac{1}{|B_i|} \sum_{(\boldsymbol{x},\boldsymbol{y}) \in B_i} \nabla_{\tilde{\mathbf{X}}}\,\Big\langle \nabla_{\tilde{\boldsymbol{\theta}}_i} \ell\!\big(f_{\tilde{\boldsymbol{\theta}}_i}(\boldsymbol{x}), \boldsymbol{y}\big), \boldsymbol{g}_L \Big\rangle$
    \STATE $\tilde{\mathbf{X}} \leftarrow \tilde{\mathbf{X}} - \eta_x\,\nabla_{\tilde{\mathbf{X}}}\mathcal{L}_{\text{BTM}}$ \COMMENT{Update synthetic inputs}
\ENDFOR
\STATE \textbf{return} $\tilde{\mathcal{D}}^\star = \tilde{\mathcal{D}}$
\end{algorithmic}
\end{algorithm}

\section{PROOF OF THEOREM~\ref{thm:bez-sur}}
\label{sec:thm1-proof}

This section provides the complete proof of Theorem~\ref{thm:bez-sur} presented in Section~\ref{subsec:bez-theory} of the main paper, establishing that optimised Bézier curves serve as effective surrogates for SGD trajectories in dataset condensation. 

\begin{proof}
    \textbf{(i) Average loss along the Bézier path is near-optimal.}  
    Since $\mathcal{L}$ is $\beta$-smooth, we have:
    \begin{equation}
        \mathcal{L}(\Phi_{\boldsymbol{\phi}^\star}(t)) \le \mathcal{L}(\gamma(t)) + \frac{\beta}{2} \|\Phi_{\boldsymbol{\phi}^\star}(t) - \gamma(t)\|^2.
    \end{equation}
    
    From Bézier interpolation theory \citep{pottmann2001}, the deviation between the Bézier curve and the piecewise-linear trajectory at parameter $t$ satisfies:
    \begin{equation}
        \|\Phi_{\boldsymbol{\phi}^\star}(t) - \gamma(t)\| \leq \frac{t(1 - t)}{2} \cdot \sup_u \|\Phi_{\boldsymbol{\phi}^\star}''(u) - \gamma''(u)\|.
    \end{equation}
    
    Since the Bézier curve has constant curvature $\|\Phi_{\boldsymbol{\phi}^\star}''(u)\| = \kappa$ (not dependent on $u$) and the SGD trajectory $\gamma''(u)$ is noisy, we conservatively bound the supremum difference by $\kappa$. Also, since $t(1-t)$ attains its maximum value $\frac{1}{4}$ at $t=\frac{1}{2}$, we have
    \begin{equation}
        0 \;\le\; t(1-t) \;\le\; \frac{1}{4}
        \qquad\Longrightarrow\qquad
        \|\Phi_{\boldsymbol{\phi}^\star}(t)-\gamma(t)\|
        \;\le\;
        \frac{t(1-t)}{2}\,\kappa.
    \end{equation}
    
    \noindent Squaring the deviation gives
    \begin{equation}
        \|\Phi_{\boldsymbol{\phi}^\star}(t)-\gamma(t)\|^{2}
        \;\le\;
        \frac{t^{2}(1-t)^{2}\,\kappa^{2}}{4}.
    \end{equation}
    
    \noindent
    Substituting this bound into the $\beta$-smoothness inequality yields
    \begin{equation}
        \mathcal{L}(\Phi_{\boldsymbol{\phi}^\star}(t))
        \;\le\;
        \mathcal{L}(\gamma(t))
        +\frac{\beta}{2}\,
        \frac{t^{2}(1-t)^{2}\,\kappa^{2}}{4}
        \;=\;
        \mathcal{L}(\gamma(t))
        +\frac{\beta\,\kappa^{2}}{8}\,t^{2}(1-t)^{2}.
    \end{equation}
    
    \noindent%
    Finally, integrate over $t\in[0,1]$:
    \begin{equation}
        \int_{0}^{1}\mathcal{L}(\Phi_{\boldsymbol{\phi}^\star}(t))\,dt
        \;\le\;
        \int_{0}^{1}\mathcal{L}(\gamma(t))\,dt
        +
        \frac{\beta\,\kappa^{2}}{8}
        \int_{0}^{1} t^{2}(1-t)^{2}\,dt.
    \end{equation}
    
    A direct calculation shows
    $\displaystyle \int_{0}^{1} t^{2}(1-t)^{2}\,dt = \frac{1}{30}$, so
    \begin{equation}
        \int_{0}^{1}\mathcal{L}(\Phi_{\boldsymbol{\phi}^\star}(t))\,dt
        \;\le\;
        \int_{0}^{1}\mathcal{L}(\gamma(t))\,dt
        +\frac{\beta\,\kappa^{2}}{240}.
    \end{equation}
    This completes the proof of part~(i).

    \medskip
    \medskip
    \noindent \textbf{(ii) Bézier path has lower and noise-free curvature compared to SGD trajectory.}  
    The quadratic Bézier curve $\Phi_{\boldsymbol{\phi}^\star}(t)$ has constant second derivative given by:
    \begin{equation}
        \Phi_{\boldsymbol{\phi}^\star}''(t) = 2(\boldsymbol{\theta}_0 - 2\boldsymbol{\phi}^\star + \boldsymbol{\theta}_T),
    \end{equation}
    
    so the curvature magnitude along the Bézier path is
    \begin{equation}
        \sup_{t \in [0,1]} \|\Phi_{\boldsymbol{\phi}^\star}''(t)\| = \kappa.
    \end{equation}
    
    \noindent To approximate curvature along the discrete SGD trajectory, we adopt the standard second-order finite difference:
    \begin{equation}
        \Delta_k := \boldsymbol{\theta}_{k+1} - 2\boldsymbol{\theta}_k + \boldsymbol{\theta}_{k-1}.
    \end{equation}
    
    \noindent Using the SGD update rule:
    \begin{equation}
        \boldsymbol{\theta}_{k+1} = \boldsymbol{\theta}_k - \eta_k \left( \nabla \mathcal{L}(\boldsymbol{\theta}_k) + \boldsymbol{\xi}_k \right),
    \end{equation}
    
    where $\eta_k$ is the learning rate and $\boldsymbol{\xi}_k$ is the stochastic gradient noise due to mini-batching, we expand:
    \begin{align}
        \Delta_k &= \boldsymbol{\theta}_{k+1} - 2\boldsymbol{\theta}_k + \boldsymbol{\theta}_{k-1} \\
        &= \left( \boldsymbol{\theta}_k - \eta_k (\nabla \mathcal{L}(\boldsymbol{\theta}_k) + \boldsymbol{\xi}_k) \right)
        - 2\boldsymbol{\theta}_k
        + \left( \boldsymbol{\theta}_k + \eta_{k-1} (\nabla \mathcal{L}(\boldsymbol{\theta}_{k-1}) + \boldsymbol{\xi}_{k-1}) \right) \\
        &= -\eta_k (\nabla \mathcal{L}(\boldsymbol{\theta}_k) + \boldsymbol{\xi}_k)
        + \eta_{k-1} (\nabla \mathcal{L}(\boldsymbol{\theta}_{k-1}) + \boldsymbol{\xi}_{k-1}).
    \end{align}
    
    \noindent
    If we assume the gradients vary slowly compared to the noise, or average out across steps (e.g., near convergence), we can approximate:
    \begin{equation}
        \Delta_k \approx -\eta_k \boldsymbol{\xi}_k + \eta_{k-1} \boldsymbol{\xi}_{k-1},
    \end{equation}
    
    highlighting the role of stochastic noise in introducing curvature into the discrete SGD path.
    
    Assuming unbiased noise and using standard deviation $\sigma_{\mathrm{sgd}}^2 = \mathbb{E}[\|\boldsymbol{\xi}_k\|^2]$, we treat $\gamma''(t)$ as a signed measure as per \cite{moulines2011}. Then the expected curvature of the SGD trajectory satisfies:
    \begin{equation}
        \mathbb{E}\left[\sup_t \|\gamma''(t)\|\right] \ge \mathbb{E}\left[\sum_k \|\Delta_k\|\right] \ge c \cdot \sigma_{\mathrm{sgd}},
    \end{equation}
    
    for some constant $c > 0$ depending on the learning rate schedule.
    
    \noindent Hence, we obtain:
    \begin{equation}
        \mathbb{E}\left[\sup_t \|\gamma''(t)\|\right] \ge \kappa + c \cdot \sigma_{\mathrm{sgd}}.
    \end{equation}
    
    \noindent
    In summary, the curvature of the optimised Bézier path and the expected curvature of the SGD trajectory satisfy:
    \begin{equation}
        \boxed{
        \sup_{t} \|\Phi_{\boldsymbol{\phi}^\star}''(t)\| = \kappa 
        \;\;<\;\; 
        \mathbb{E}\left[\sup_t \|\gamma''(t)\|\right] \approx \kappa + c \cdot \sigma_{\mathrm{sgd}}.
        }
    \end{equation}
    
    This inequality highlights that the Bézier path provides a smoother alternative to SGD by avoiding the curvature amplification caused by stochastic gradient noise.

    \medskip
    \medskip
    \noindent \textbf{(iii) Model predictions along Bézier path remain close to those along SGD.}  
    Assume the model map $f_{\boldsymbol{\theta}}(\boldsymbol{x})$ is $L_f$-Lipschitz continuous in $\boldsymbol{\theta}$ for every $\boldsymbol{x} \in \mathcal{X}$, i.e.,
    \begin{equation}
        \|f_{\boldsymbol{\theta}_1}(\boldsymbol{x}) - f_{\boldsymbol{\theta}_2}(\boldsymbol{x})\| \le L_f \|\boldsymbol{\theta}_1 - \boldsymbol{\theta}_2\| \quad \forall \boldsymbol{\theta}_1, \boldsymbol{\theta}_2 \in \Theta.
    \end{equation}
    
    \noindent
    Applying this to $\boldsymbol{\theta}_1 = \Phi_{\boldsymbol{\phi}^\star}(t)$ and $\boldsymbol{\theta}_2 = \gamma(t)$ yields:
    \begin{equation}
        \|f_{\Phi_{\boldsymbol{\phi}^\star}(t)}(\boldsymbol{x}) - f_{\gamma(t)}(\boldsymbol{x})\| \le L_f \|\Phi_{\boldsymbol{\phi}^\star}(t) - \gamma(t)\|.
    \end{equation}
    
    \noindent
    From the curvature bound derived earlier (see part (i)), we already have:
    \begin{equation}
        \|\Phi_{\boldsymbol{\phi}^\star}(t) - \gamma(t)\| \le \frac{t(1 - t)}{2} \kappa.
    \end{equation}
    
    \noindent
    Since $t(1 - t) \le \frac{1}{4}$ for all $t \in [0, 1]$, it follows that:
    \begin{equation}
        \|\Phi_{\boldsymbol{\phi}^\star}(t) - \gamma(t)\| \le \frac{\kappa}{8},
    \end{equation}
    
    and therefore,
    \begin{equation}
        \|f_{\Phi_{\boldsymbol{\phi}^\star}(t)}(\boldsymbol{x}) - f_{\gamma(t)}(\boldsymbol{x})\| \le \frac{L_f \kappa}{8},
    \end{equation}
    
    for all $\boldsymbol{x} \in \mathcal{X}$ and $t \in [0, 1]$. This establishes that the model predictions along the Bézier path remain uniformly close to those along the SGD trajectory.

    This completes the proof.
\end{proof}

\section{DATASET DETAILS}
\label{sec:dataset_details}

Across all datasets, we employed a 70/15/15 split for training, validation, and test sets. Continuous features were z-score normalised using training-set statistics. The training set was used for condensation, the validation set for hyperparameter and model selection, and the test set was held out for final evaluation.

\subsection{CURIAL}

The CURIAL database is an anonymised database with United Kingdom National Health Service (NHS) approval via the national oversight/regulatory body, the Health Research Authority (HRA) (CURIAL; NHS HRA IRAS ID: 281832). 

Data from Oxford University Hospitals (OUH) studied here are available from the Infections in Oxfordshire Research Database\footnote{\url{https://oxfordbrc.nihr.ac.uk/research-themes/modernising-medical-microbiology-and-big-infection-diagnostics/infections-in-oxfordshire-research-database-iord/}}, subject to an application meeting the ethical and governance requirements of the Database. Data from University Hospital Birmingham (UHB) and Portsmouth University Hospitals (PUH) are available on reasonable request to the respective trusts, subject to HRA requirements.

\begin{table}[ht]
\centering
\caption{Dataset characteristics for CURIAL sites.}
\label{tab:curial_examples}
\begin{tabular}{lccc}
\toprule
                        & Oxford  & Portsmouth & Birmingham \\ 
\midrule
\# Examples             & 161,955 & 38,717     & 95,236     \\
\# Positive Examples    & 2,791   & 2,005      & 790        \\ 
\# Features             & 27      & 27         & 27         \\ 
Prevalence (\%)         & 1.7     & 5.3        & 0.8        \\
\bottomrule    
\end{tabular}
\end{table}

\begin{table}[ht]
\centering
\caption{Clinical predictors considered for COVID-19 status prediction.}
\label{tab:curial_features}
\scalebox{0.9}{
\begin{tabular}{p{0.43\textwidth}p{0.5\textwidth}}\toprule
\multicolumn{1}{c}{\textbf{Category}} & \multicolumn{1}{c}{\textbf{Features}} \\ \midrule
Vital Signs & Heart rate, respiratory rate, oxygen saturation, systolic blood pressure, diastolic blood pressure, temperature \\ 
\cmidrule{2-2}
Blood Tests & Haemoglobin, haematocrit, mean cell volume, white cell count, neutrophil count, lymphocyte count, monocyte count, eosinophil count, basophil count, platelets \\ 
\cmidrule{2-2}
Liver Function Tests \& C-reactive protein & Albumin, alkaline phosphatase, alanine aminotransferase, bilirubin, C-reactive protein \\ 
\cmidrule{2-2}
Urea \& Electrolytes & Sodium, potassium, creatinine, urea, estimated glomerular filtration rate \\ \bottomrule
\end{tabular}
}
\end{table}

\noindent Tables~\ref{tab:curial_features} and \ref{tab:curial_examples} document the features and dataset characteristics at each CURIAL site. Median imputation was applied to missing values in CURIAL datasets.

\subsection{eICU}

The eICU Collaborative Research Database (eICU-CRD) is a publicly available multi-centre EHR dataset~\citep{pollard2018eicu}. We used a pre-processed version available at \url{https://physionet.org/content/mimic-eicu-fiddle-feature/1.0.0/}. The dataset comprises 49,305 adult ICU stays with 402 features per sample. Features were derived through quantile-based binning of continuous variables and one-hot encoding of categorical variables.

\begin{table}[ht]
\centering
\caption{Dataset characteristics for eICU.}
\label{tab:eicu_examples}
\begin{tabular}{lc}
\toprule
                        & eICU    \\ 
\midrule
\# Examples             & 49,305  \\
\# Positive Examples    & 4,501   \\ 
\# Features             & 402     \\ 
Prevalence (\%)         & 9.1     \\
\bottomrule    
\end{tabular}
\end{table}

\begin{table}[ht]
\centering
\caption{Source physiological and demographic variables used to construct the 402-dimensional feature vector for eICU IHM prediction.}
\label{tab:eicu_features}
\scalebox{0.85}{
\begin{tabular}{p{0.43\textwidth}p{0.5\textwidth}}\toprule
\multicolumn{1}{c}{\textbf{Category}} & \multicolumn{1}{c}{\textbf{Variables}} \\ \midrule
Demographics & Age, height, weight, gender \\
\cmidrule{2-2}
Admission Information & Hospital admit source, hospital admit offset, unit admit source, unit stay type, unit type, Apache admission diagnosis, airway type \\
\cmidrule{2-2}
Vital Signs & Heart rate, respiratory rate, oxygen saturation, temperature (Celsius and Fahrenheit), temperature location \\
\cmidrule{2-2}
Blood Pressure & Non-invasive systolic/diastolic/mean blood pressure, invasive systolic/diastolic/mean blood pressure, central venous pressure \\
\cmidrule{2-2}
Oxygen Support & O2 administration device, O2 level percentage \\
\cmidrule{2-2}
Laboratory Measurements & Glucose \\
\bottomrule
\end{tabular}
}
\end{table}

\noindent Continuous variables were binned into quantiles (typically 5 bins) and summary statistics (minimum, maximum, mean) were computed where applicable. Categorical variables were one-hot encoded. This preprocessing resulted in the 402-dimensional feature vector used for model training.

\subsection{MIMIC-III}

The MIMIC-III database is a large publicly available ICU dataset from Beth Israel Deaconess Medical Center~\citep{johnson2016mimic}. We processed the dataset using publicly available benchmarking code, removing duplicate features to obtain a 60-dimensional feature vector at each time-step across 48 hourly time-steps. The training set was class-balanced to address the imbalanced nature of in-hospital mortality prediction.

\begin{table}[ht]
\centering
\caption{Dataset characteristics for MIMIC-III.}
\label{tab:mimic_examples}
\begin{tabular}{lc}
\toprule
                                & MIMIC-III \\ 
\midrule
\# Examples                     & 21,156    \\
\# Time-steps                   & 48        \\
\# Features per time-step       & 60        \\ 
Prevalence - Train (\%)         & 50.0      \\
Prevalence - Validation (\%)    & 13.5      \\
Prevalence - Test (\%)          & 11.5      \\
\bottomrule    
\end{tabular}
\end{table}

\begin{table}[ht]
\centering
\caption{Source physiological and demographic variables used to construct the 60-dimensional feature vector at each time-step for MIMIC-III IHM prediction.}
\label{tab:mimic_features}
\scalebox{0.85}{
\begin{tabular}{p{0.43\textwidth}p{0.5\textwidth}}\toprule
\multicolumn{1}{c}{\textbf{Category}} & \multicolumn{1}{c}{\textbf{Variables}} \\ \midrule
Demographics & Height, weight \\
\cmidrule{2-2}
Vital Signs & Heart rate, respiratory rate, temperature, oxygen saturation, capillary refill rate \\
\cmidrule{2-2}
Blood Pressure & Systolic blood pressure, diastolic blood pressure, mean blood pressure \\
\cmidrule{2-2}
Respiratory Support & Fraction inspired oxygen \\
\cmidrule{2-2}
Neurological Assessment & Glasgow Coma Scale (eye opening, motor response, verbal response, total score) \\
\cmidrule{2-2}
Laboratory Measurements & Glucose, pH \\
\bottomrule
\end{tabular}
}
\end{table}

\noindent Glasgow Coma Scale components were one-hot encoded (eye opening: 5 categories; motor response: 6 categories; verbal response: 5 categories; total score: 11 categories). Capillary refill rate was encoded as a binary variable. Binary mask features indicate the presence or absence of measurements at each time-step. This preprocessing resulted in 60 features per time-step, capturing both the physiological measurements and their availability across the 48-hour observation window.

\section{IMPLEMENTATION DETAILS}
\label{sec:implementation}

\subsection{Model Architectures}
\label{subsec:architectures}
\textsc{CURIAL and eICU:} A multi-layer perceptron (MLP) \citep{rumelhart1986learning} with a single hidden layer of $h$ nodes was used, followed by ReLU activation, an output layer, and sigmoid activation. Dropout with rate 0.25 was applied after the hidden layer during training. For eICU, $h = 256$ hidden nodes were used; for CURIAL datasets, $h = 64$ hidden nodes were used.

To evaluate the generalisability of the condensed eICU dataset across different network architectures, several variants of the base MLP were tested:
\begin{itemize}[noitemsep, topsep=0pt]
    \item MLP-1: Doubles the hidden layer width to $2h$ nodes
    \item MLP-2: Quadruples the hidden layer width to $4h$ nodes
    \item MLP-3: Adds a second hidden layer with $h$ nodes in the first layer and $2h$ nodes in the second layer
    \item MLP-4: Uses three hidden layers with $h$, $2h$, and $4h$ nodes respectively
\end{itemize}
All variants maintain the same dropout rate (0.25) applied after each hidden layer and activation function (ReLU) as the base architecture.

\textsc{MIMIC-III:} A multi-scale multi-branch temporal convolutional network (TCN) architecture \citep{bai2018empirical} with 192 output channels was used. This consisted of 3 parallel branches with kernel sizes of [3, 5, 7], each producing 64 channels, with PReLU activation and dropout rate 0.5. The network processes 48 time-steps of 60-dimensional feature vectors using causal dilated convolutions with residual connections. The output is averaged across the temporal dimension, followed by a linear layer with sigmoid activation.

To evaluate cross-architecture generalisability of the condensed MIMIC-III dataset, the following TCN and long short-term memory (LSTM) \citep{hochreiter1997long} networks were tested:
\begin{itemize}[noitemsep, topsep=0pt]
    \item TCN-1: Single-scale TCN with kernel size 9, 64 channels, PReLU activation, and dropout rate 0.75
    \item TCN-2: Two-layer multi-scale TCN with kernel sizes of [3, 5], 256 output channels per layer, PReLU activation, and dropout rate 0.5
    \item LSTM-1: Single-layer LSTM with 128 hidden units, followed by dropout (0.25) and a linear output layer with sigmoid activation
    \item LSTM-2: Single-layer LSTM with 256 hidden units, followed by dropout (0.25) and a linear output layer with sigmoid activation
\end{itemize}

\subsection{Training Trajectories}
All hyperparameters were selected using the validation set for each dataset through grid search. Fifty training trajectories were generated for each dataset by training the backbone networks described in Section \ref{subsec:architectures} from different random initialisations. 

For MTT and TESLA, standard SGD optimisation was employed with learning rates of 0.02 for MIMIC-III, eICU, OUH, and UHB, and 0.01 for PUH. Momentum was set to 0.0 for MIMIC-III and 0.9 for all other datasets, with training conducted for 60 epochs on MIMIC-III and 100 epochs on all other datasets. 

FTD trajectories were trained using the Generalized Sharpness-Aware Minimization (GSAM) optimiser. A linear learning rate scheduler was employed, decaying from the maximum learning rate (matching the corresponding SGD runs) to 0 over $t_{\text{max}}$ steps, where $t_{\text{max}}$ was defined as the product of the number of epochs and the number of batches per epoch. The proportional scheduler for the perturbation radius scaled from 1 to 0 as the learning rate decayed, coupling the perturbation magnitude to the learning rate schedule.

DATM employed FTD trajectories (with GSAM) for all tabular datasets (eICU, OUH, PUH, UHB), but used standard MTT trajectories (with SGD) for MIMIC-III, as FTD demonstrated worse performance than MTT on this sequential dataset.

\subsection{Control Point Optimisation}

For BTM, quadratic Bézier curves were optimised between each pair of SGD trajectory endpoints (the same trajectories used for MTT and TESLA) following Algorithm \ref{alg:control_optimization}. The control point $\boldsymbol{\phi}$ was initialised at the parameter space midpoint between $\boldsymbol{\theta}_0$ and $\boldsymbol{\theta}_T$, ensuring the initial Bézier curve lies within the convex hull of the endpoints.

\begin{algorithm}[H]
\caption{Control Point Optimisation}
\label{alg:control_optimization}
\begin{algorithmic}[1]
    \REQUIRE SGD trajectory endpoints $\boldsymbol{\theta}_0, \boldsymbol{\theta}_T$, dataset $\mathcal{D}$, learning rate $\eta_{\phi}$, convergence tolerance $\epsilon$, maximum iterations $T_{\max}$, Monte Carlo samples $N_{MC}$
    \ENSURE Optimised control point $\boldsymbol{\phi}^*$
    \STATE Initialise $\boldsymbol{\phi} \leftarrow \frac{\boldsymbol{\theta}_0 + \boldsymbol{\theta}_T}{2}$ \COMMENT{Midpoint initialisation}
    \STATE $t \leftarrow 0$
    \WHILE{$t < T_{\max}$}
        \STATE Sample $\{t_1, t_2, \ldots, t_{N_{MC}}\} \sim \mathcal{U}(0,1)$ \COMMENT{Monte Carlo samples}
        \STATE $\mathcal{L}_{\text{avg}} \leftarrow 0$
        \FOR{$i = 1$ to $N_{MC}$}
            \STATE $\boldsymbol{\theta}_{t_i} \leftarrow (1-t_i)^2 \boldsymbol{\theta}_0 + 2t_i(1-t_i)\boldsymbol{\phi} + t_i^2 \boldsymbol{\theta}_T$
            \STATE Sample mini-batch $\mathcal{B} \subset \mathcal{D}$
            \STATE $\mathcal{L}_{\text{avg}} \leftarrow \mathcal{L}_{\text{avg}} + \frac{1}{N_{MC}} \mathcal{L}_{\mathrm{CE}}(f_{\boldsymbol{\theta}_{t_i}}, \mathcal{B})$
        \ENDFOR
        \STATE $\boldsymbol{g} \leftarrow \nabla_{\boldsymbol{\phi}} \mathcal{L}_{\text{avg}}$ \COMMENT{Compute gradient w.r.t. control point}
        \IF{$\|\boldsymbol{g}\|_2 < \epsilon$}
            \STATE \textbf{break} \COMMENT{Convergence achieved}
        \ENDIF
        \STATE $\boldsymbol{\phi} \leftarrow \boldsymbol{\phi} - \eta_{\phi} \boldsymbol{g}$ \COMMENT{Gradient descent update}
        \STATE $t \leftarrow t + 1$
    \ENDWHILE
    \STATE \textbf{return} $\boldsymbol{\phi}^* = \boldsymbol{\phi}$
\end{algorithmic}
\end{algorithm}

The gradient with respect to the control point was computed via automatic differentiation through the Bézier parameterisation using the chain rule:
\begin{equation}
    \nabla_{\boldsymbol{\phi}} \mathcal{L}_{\text{avg}} = \frac{1}{N_{MC}} \sum_{i=1}^{N_{MC}} \nabla_{\boldsymbol{\theta}} \mathcal{L}_{\mathrm{CE}}(f_{\boldsymbol{\theta}_{t_i}}, \mathcal{B}) \cdot \frac{\partial \boldsymbol{\theta}_{t_i}}{\partial \boldsymbol{\phi}}
\end{equation}
where $\frac{\partial \boldsymbol{\theta}_{t_i}}{\partial \boldsymbol{\phi}} = 2t_i(1-t_i)$ from the Bézier curve formula.

\textbf{Hyperparameters.}
The following hyperparameters were used for control point optimisation: learning rate $\eta_{\phi} = 10^{-2}$, convergence tolerance $\epsilon = 10^{-5}$, maximum iterations $T_{\max} = 300$, and Monte Carlo samples $N_{MC}=2$. Convergence typically occurred within 200-300 iterations depending on the complexity of the loss landscape, with the algorithm terminating when $\|\boldsymbol{g}\|_2 < \epsilon$ or after $T_{\max}$ iterations.

\textbf{Computational Cost.}
The control point optimisation requires $2 \times T_{\max}$ forward passes per trajectory (corresponding to $N_{MC} = 2$ Monte Carlo samples at each iteration). For a training dataset with $|\mathcal{D}_{\text{train}}|$ examples and batch size $b$, a single training epoch requires $|\mathcal{D}_{\text{train}}|/b$ forward passes. Thus, optimising one control point at the maximum $T_{\max} = 300$ iterations is equivalent to approximately $\frac{2 \times 300}{|\mathcal{D}_{\text{train}}|/b} = \frac{600b}{|\mathcal{D}_{\text{train}}|}$ epochs of training. All datasets use batch size $b = 256$ and a 70/15/15 train/validation/test split. For OUH with $|\mathcal{D}_{\text{train}}| = 0.7 \times 161{,}955 \approx 113{,}369$ examples, this amounts to $\frac{600 \times 256}{113{,}369} \approx 1.4$ equivalent epochs per trajectory. Similarly, for UHB ($|\mathcal{D}_{\text{train}}| \approx 66{,}665$), PUH ($|\mathcal{D}_{\text{train}}| \approx 27{,}102$), eICU ($|\mathcal{D}_{\text{train}}| \approx 34{,}514$), and MIMIC-III ($|\mathcal{D}_{\text{train}}| \approx 14{,}809$), the costs are approximately $2.3$, $5.7$, $4.5$, and $10.4$ equivalent epochs per trajectory, respectively. Since control point optimisation is performed once per trajectory during the initial surrogate construction phase (prior to dataset condensation), this one-off computational overhead is amortised across all subsequent condensation iterations.

\subsection{Dataset Condensation}

\textbf{BTM Hyperparameters.}
Our proposed BTM method uses trajectory segments of length 0.2, with starting points sampled as $t_{\text{start}} \sim \mathcal{U}(0, 0.8)$ and ending points set to $t_{\text{end}} = t_{\text{start}} + 0.2$. The student model takes $N = 30$ gradient steps with learning rate $\eta_s = 0.01$. The synthetic data is optimised using SGD with meta learning rate $\eta_x = 100$ and momentum $0.9$. The student learning rate $\eta_s$ is itself meta-optimised with respect to the synthetic data using learning rate $10^{-4}$ and momentum $0.5$. The synthetic batch size is set to $b = \max(2 \times \text{ipc}, 256)$, and the total number of condensation iterations is $T_{\max} = 40{,}000$.

\textbf{Baseline Configurations.}
For comparison, MTT, FTD, and TESLA matched along the entire trajectory with $M = 5$ expert epochs and $N = 80$ student gradient steps for all datasets. DATM employed sequential generation with $M = 5$, where $T^-$, $T$, and $T^+$ denote the lower bound, current upper bound, and final upper bound of the trajectory sampling range, respectively. For eICU and CURIAL datasets, the configurations were: $(T^-, T, T^+) = (0, 20, 40)$ for 50 \emph{ipc}, $(10, 20, 60)$ for 100 \emph{ipc}, and $(20, 40, 100)$ for 200 and 500 \emph{ipc}. For MIMIC-III, the configurations were: $(0, 10, 25)$ for 50 \emph{ipc}, $(5, 15, 40)$ for 100 \emph{ipc}, $(20, 40, 50)$ for 200 \emph{ipc}, and $(30, 40, 60)$ for 500 \emph{ipc}. All baseline methods used $N = 80$ student steps and shared the same synthetic data optimisation settings as BTM. All methods used hard binary labels except TESLA, which generated soft labels from the target teacher model $\boldsymbol{\theta}_{t+M}$ at each condensation iteration.

For the M3D baseline, the trajectory matching backbone networks were used without sigmoid activation, with the output size changed from $1$ to $62$ nodes for CURIAL, $256$ for eICU, and $128$ for MIMIC-III.

\textbf{Evaluation Protocol.}
During condensation, the synthetic dataset was evaluated every 10 iterations by training a randomly initialised network and saved when validation AUPRC improved. Experiments saving based on AUROC found that the best AUPRC always corresponded to high AUROC, but the reverse was not true, for reasons discussed in the main paper. The evaluation hyperparameters were: for MIMIC-III, learning rate $0.02$, momentum $0.9$, and $60$ epochs; for eICU, OUH, PUH, and UHB, learning rate $0.05$, momentum $0.9$, and $50$ epochs.

\subsection{Evaluation}

\textbf{Final Evaluation Protocol.}
All methods were evaluated using the trajectory-matching backbone networks described in Section~\ref{subsec:architectures}. Models were trained on the condensed synthetic datasets with the following hyperparameters: for MIMIC-III, learning rate $0.02$, momentum $0.9$, and $80$ epochs; for eICU, OUH, PUH, and UHB, learning rate $0.05$, momentum $0.9$, and $100$ epochs.

\textbf{Main Results.}
The primary evaluation used the MLP and TCN architectures across all datasets to assess the quality of the condensed synthetic datasets.

\textbf{Cross-Architecture Generalisation.}
To evaluate the transferability of condensed datasets across different architectures, we conducted cross-architecture experiments. For eICU, the MLP variants (MLP-1 through MLP-4) were evaluated using the same hyperparameters as the base MLP (learning rate $0.05$, momentum $0.9$, $100$ epochs). For MIMIC-III, the TCN variants (TCN-1, TCN-2) used the same hyperparameters as the base TCN (learning rate $0.02$, momentum $0.9$, $80$ epochs), while the LSTM variants (LSTM-1, LSTM-2) were trained with learning rate $0.01$, momentum $0.9$, and $30$ epochs.




\section{ADDITIONAL RESULTS}
\label{sec:add_res}

\begin{table}
\centering
\caption{
Performance on eICU and MIMIC-III datasets across different \emph{ipc} levels. Best results at each \emph{ipc} are highlighted in \textcolor{blue}{blue} (ours) and \textcolor{red}{red} (baseline).}
\label{tab:extra_res}

\textbf{(a) eICU}\\[0.5ex]
\begin{sc}
\resizebox{0.9\linewidth}{!}{
    \begin{tabular}{lccccBcccc}
    \toprule
    & \multicolumn{4}{cB}{\textbf{AUROC}} & \multicolumn{4}{c}{\textbf{AUPRC}} \\
    \cmidrule(lr{0.5em}){2-5} \cmidrule(lr{0.5em}){6-9}
    \textbf{Method} & \textbf{50} & \textbf{100} & \textbf{200} & \textbf{500} &
    \textbf{50} & \textbf{100} & \textbf{200} & \textbf{500} \\
    \midrule
    Random     & $0.740_{\pm 0.010}$ & $0.763_{\pm 0.009}$ & $0.804_{\pm 0.008}$ & $0.840_{\pm 0.005}$
               & $0.277_{\pm 0.018}$ & $0.308_{\pm 0.020}$ & $0.368_{\pm 0.018}$ & $0.429_{\pm 0.013}$ \\
    M3D        & $0.779_{\pm 0.002}$ & $0.824_{\pm 0.002}$ & $0.846_{\pm 0.001}$ & $0.852_{\pm 0.001}$
               & $0.334_{\pm 0.003}$ & $0.400_{\pm 0.003}$ & $0.447_{\pm 0.002}$ & $0.435_{\pm 0.005}$ \\
    MTT        & $0.768_{\pm 0.004}$ & $0.793_{\pm 0.002}$ & $0.824_{\pm 0.003}$ & $0.845_{\pm 0.002}$
               & $0.336_{\pm 0.004}$ & $0.369_{\pm 0.004}$ & $0.388_{\pm 0.006}$ & $0.450_{\pm 0.004}$ \\
    TESLA      & $0.740_{\pm 0.004}$ & $0.796_{\pm 0.001}$ & $0.815_{\pm 0.001}$ & $0.839_{\pm 0.001}$
               & $0.312_{\pm 0.005}$ & $0.343_{\pm 0.003}$ & $0.377_{\pm 0.004}$ & $0.436_{\pm 0.002}$ \\
    FTD        & $0.802_{\pm 0.002}$ & $0.798_{\pm 0.002}$ & $0.816_{\pm 0.002}$ & $0.849_{\pm 0.002}$
               & $0.378_{\pm 0.008}$ & $0.382_{\pm 0.003}$ & $0.422_{\pm 0.003}$ & $0.444_{\pm 0.003}$ \\
    DATM       & $0.852_{\pm 0.002}$ & $0.853_{\pm 0.002}$ & $0.850_{\pm 0.002}$ & $0.854_{\pm 0.001}$
               & $0.459_{\pm 0.002}$ & $0.448_{\pm 0.002}$ & $0.461_{\pm 0.002}$ & $0.477_{\pm 0.002}$ \\
    \rowcolor{btmgray}
    BTM (Ours) & \bestOurs{$0.858_{\pm 0.004}$} & \bestOurs{$0.856_{\pm 0.004}$} & \bestOurs{$0.858_{\pm 0.003}$} & \bestOurs{$0.871_{\pm 0.001}$}
               & \bestOurs{$0.466_{\pm 0.007}$} & \bestOurs{$0.479_{\pm 0.008}$} & \bestOurs{$0.473_{\pm 0.007}$} & \bestOurs{$0.504_{\pm 0.003}$} \\
    \midrule
    \emph{\textbf{Full Dataset}} & \multicolumn{4}{cB}{$\mathbf{0.879_{\pm 0.002}}$} & \multicolumn{4}{c}{$\mathbf{0.515_{\pm 0.003}}$} \\
    \bottomrule
    \end{tabular}
}
\end{sc}
\par\vspace{3ex}

\textbf{(b) MIMIC-III}\\[0.5ex]
\begin{sc}
\resizebox{0.9\linewidth}{!}{
    \begin{tabular}{lccccBcccc}
    \toprule
    & \multicolumn{4}{cB}{\textbf{AUROC}} & \multicolumn{4}{c}{\textbf{AUPRC}} \\
    \cmidrule(lr{0.5em}){2-5} \cmidrule(lr{0.5em}){6-9}
    \textbf{Method} & \textbf{50} & \textbf{100} & \textbf{200} & \textbf{500} &
    \textbf{50} & \textbf{100} & \textbf{200} & \textbf{500} \\
    \midrule
    Random     & $0.743_{\pm 0.014}$ & $0.751_{\pm 0.016}$ & $0.776_{\pm 0.008}$ & $0.789_{\pm 0.009}$ 
               & $0.344_{\pm 0.028}$ & $0.360_{\pm 0.027}$ & $0.383_{\pm 0.024}$ & $0.399_{\pm 0.014}$ \\
    M3D        & $0.771_{\pm 0.006}$ & $0.774_{\pm 0.004}$ & $0.788_{\pm 0.010}$ & $0.799_{\pm 0.008}$ 
               & $0.390_{\pm 0.007}$ & $0.401_{\pm 0.006}$ & $0.413_{\pm 0.007}$ & $0.403_{\pm 0.017}$ \\
    MTT        & $0.805_{\pm 0.006}$ & $0.823_{\pm 0.003}$ & $0.833_{\pm 0.002}$ & $0.836_{\pm 0.002}$ 
               & $0.431_{\pm 0.020}$ & $0.456_{\pm 0.013}$ & $0.483_{\pm 0.006}$ & $0.488_{\pm 0.008}$ \\
    TESLA      & $0.788_{\pm 0.006}$ & $0.793_{\pm 0.006}$ & $0.800_{\pm 0.010}$ & $0.808_{\pm 0.006}$ 
               & $0.385_{\pm 0.006}$ & $0.395_{\pm 0.008}$ & $0.413_{\pm 0.015}$ & $0.418_{\pm 0.005}$ \\
    FTD        & $0.798_{\pm 0.004}$ & $0.813_{\pm 0.006}$ & $0.815_{\pm 0.006}$ & $0.823_{\pm 0.004}$ 
               & $0.434_{\pm 0.009}$ & $0.438_{\pm 0.018}$ & $0.448_{\pm 0.013}$ & $0.467_{\pm 0.009}$ \\
    DATM       & $0.824_{\pm 0.003}$ & $0.834_{\pm 0.002}$ & $0.838_{\pm 0.002}$ & $0.837_{\pm 0.003}$ 
               & $0.467_{\pm 0.010}$ & $0.488_{\pm 0.007}$ & \bestOther{$0.500_{\pm 0.007}$} & $0.498_{\pm 0.006}$ \\
    \rowcolor{btmgray}
    BTM        & \bestOurs{$0.830_{\pm 0.005}$} & \bestOurs{$0.835_{\pm 0.004}$} & \bestOurs{$0.840_{\pm 0.001}$} & \bestOurs{$0.840_{\pm 0.001}$} 
               & \bestOurs{$0.472_{\pm 0.014}$} & \bestOurs{$0.489_{\pm 0.010}$} & $0.498_{\pm 0.005}$ & \bestOurs{$0.502_{\pm 0.003}$} \\
    \midrule
    \emph{\textbf{Full Dataset}} & \multicolumn{4}{cB}{$\mathbf{0.837_{\pm 0.003}}$} & \multicolumn{4}{c}{$\mathbf{0.499_{\pm 0.008}}$} \\
    \bottomrule
    \end{tabular}
}
\end{sc}
\end{table}

\end{document}